\documentclass[10pt]{article} % For LaTeX2e

\usepackage[preprint]{rlj} % Should be uncommented for preprint versions

\usepackage{amssymb}            % Defines common symbols like \mathbb R
\usepackage{amsthm}
\usepackage{mathtools}          % Extends amsmath, providing common math tools
\usepackage{mathrsfs}           % Enables \mathscr, which can work in cases that \mathcal does not
\usepackage{graphicx}           % For including images
\usepackage{subcaption}         % Allows for the use of subfigures and subcaptions
\usepackage[space]{grffile}     % For spaces in image names
\usepackage{url}                % For displaying URLs
\usepackage{lipsum}             % For placeholder text
\usepackage{algorithm}
\usepackage{algpseudocode}
\usepackage{comment}
\usepackage{float}
\newtheorem{assumption}{Assumption}[section]
\newtheorem{theorem}{Theorem}[section] 
\newtheorem{Lemma}{Lemma}[section]

\newtheorem{corollary}{Corollary}[section]
\allowdisplaybreaks
\title{Policy Gradient Primal-Dual Method for Safe Reinforcement Learning from Human Feedback}

\setrunningtitle{Policy Gradient Primal-Dual Method for Safe Reinforcement Learning from Human Feedback}

\author{Qiang Liu\textsuperscript{1}, Adrienne Kline \textsuperscript{1,2,3,4}, Ermin Wei\textsuperscript{1,5}}

    \emails{qiang.liu@u.northwestern.edu, \ \{adrienne.kline, ermin.wei\}@northwestern.edu}

\affiliations{
$^{1}$\textbf{Department of Electrical and Computer Engineering, Northwestern University, Evanston, IL, USA}\\
$^{2}$\textbf{Center for Artificial Intelligence, Bluhm Cardiovascular Institute, Northwestern Medicine, Chicago, IL, USA}\\
$^{3}$\textbf{Department of Surgery, Northwestern University, Chicago, IL, USA}\\
$^{4}$\textbf{Department of Radiology, Northwestern University, Chicago, IL, USA; and Xtasis Inc., Chicago, IL, USA}\\
$^{5}$\textbf{Department of Industrial Engineering and Management Sciences
Department, Northwestern University, Evanston, IL, USA}
\par % If including additional comments like below, use \par to add some whitespace. 
}

\contribution{
    We propose algorithms for safe reinforcement learning from human feedback (RLHF) with global rate guarantee, one that estimates the advantage function from human feedback within a primal dual framework and another that uses a zeroth order approach to approximate policy gradients. 
    }
    {
    Prior work \citep{dai2023safe} introduced Safe Reinforcement Learning from Human Feedback (Safe RLHF) and proposed an algorithm without rate guarantee. 
    }

\contribution{
   Our work investigates safe RLHF in the discounted infinite horizon constrained Markov decision process (CMDP) setting without relying on reward inference and establishes global convergence guarantees.
    }
    {
    Prior work \citep{zhang2024zeroth} investigates RLHF  in the episodic MDP setting without reward inference, and established convergence to a stationary policy.
    }

 % Your keywords

\summary{\emph{Safe Reinforcement Learning from Human Feedback} (Safe RLHF) has recently achieved empirical success in developing helpful and harmless large language models by decoupling human preferences regarding helpfulness and harmlessness. Existing approaches typically rely on fitting fixed horizon reward models from human feedback and have only been validated empirically. In this paper, we formulate safe RLHF  as an infinite horizon discounted Constrained Markov Decision Process (CMDP), since humans may interact with the model over a continuing sequence of interactions rather than within a single finite episode. We propose  two Safe RLHF algorithms that do not require reward model fitting and, in contrast to prior work assuming fixed-length trajectories, support flexible trajectory lengths for training. Both algorithms are based on the primal-dual method and achieve global convergence guarantees with polynomial rates in terms of policy gradient iterations, trajectory sample lengths, and human preference queries. To the best of our knowledge, this is the first work to study infinite horizon discounted CMDP under human feedback and establish global, non-asymptotic convergence.
}

\begin{document}

\makeCover  % Create the cover page
\maketitle  % Make the title section

\begin{abstract}
\emph{Safe Reinforcement Learning from Human Feedback} (Safe RLHF) has recently achieved empirical success in developing helpful and harmless large language models by decoupling human preferences regarding helpfulness and harmlessness. Existing approaches typically rely on fitting fixed horizon reward models from human feedback and have only been validated empirically. In this paper, we formulate safe RLHF  as an infinite horizon discounted Constrained Markov Decision Process (CMDP), since humans may interact with the model over a continuing sequence of interactions rather than within a single finite episode. We propose  two Safe RLHF algorithms that do not require reward model fitting and, in contrast to prior work assuming fixed-length trajectories, support flexible trajectory lengths for training. Both algorithms are based on the primal-dual method and achieve global convergence guarantees with polynomial rates in terms of policy gradient iterations, trajectory sample lengths, and human preference queries. To the best of our knowledge, this is the first work to study infinite horizon discounted CMDP under human feedback and establish global, non-asymptotic convergence.
\end{abstract}

%%%%%%%%%%%%%%%%%%%%%%%%%%%%%%%%%%%%%%%%%%%%%%%%%%%%%%%%%%%%%%%%
%% Section: Submission of papers to RLJ/RLC
%%%%%%%%%%%%%%%%%%%%%%%%%%%%%%%%%%%%%%%%%%%%%%%%%%%%%%%%%%%%%%%%
\section{Introduction}
Large language models (LLMs) have achieved remarkable success across a wide range of natural language processing tasks \citep{guo2025deepseek,achiam2023gpt}, but they face the challenge of generating inaccurate information and producing biased, misleading, or harmful responses from a human perspective \citep{bai2022training,sun2023principle}. To address these challenges, reinforcement learning from human feedback (RLHF), which learns from human preference signals collected over sampled trajectories of model interaction, has emerged as a widely adopted approach for aligning model outputs with human preferences for helpfulness and harmlessness \citep{askell2021general,bai2022training,ouyang2022training}. However, the objectives of improving helpfulness and ensuring harmlessness often come into conflict in practice \citep{ganguli2022red,bai2022training}. To address this issue, \citet{dai2023safe} introduced a novel framework called {\it Safe Reinforcement Learning from Human Feedback} (Safe RLHF), which effectively tackles the problem.

A key contribution of Safe RLHF lies in decoupling human preferences during data annotation by creating two parallel datasets from the same set of Question and Answer (QA) pairs: a {\it helpfulness}-related dataset, where responses are compared by how well they address the prompt, and a {\it harmlessness}-related dataset, where responses are compared by relative harmlessness. In addition, a binary harm label is encoded for each prompt, designating harmful responses as $``+1"$ and harmless responses as $``-1"$. 

However, their approach relies on fitting both a reward model, which evaluates how helpful or preferred a response is, and a cost model, which evaluates how unsafe or undesirable a response is; these learned models serve as surrogate signals for policy optimization but are subject to issues such as double misspecification, evaluation without ground-truth references, distributional shift, and overfitting during joint reward–policy training \citep{casper2023open}. 

 In addition to addressing the tradeoff between helpfulness and harmlessness, another challenge in RLHF is the potentially variable trajectory length. RLHF can be naturally formulated as an infinite horizon discounted MDP, where the state encodes the prompt and the evolving interaction history, and the action corresponds to the model’s next generation decision, so the resulting interaction trajectory may be arbitrarily long. Discounting further captures the intuition that immediate outcomes typically matter more than distant ones.
 However, most of the RLHF literature focuses on episodic Markov Decision Process (MDP), where the agent interacts with the environment in episodes of fixed length \citep{zhan2023provableoff,zhan2023provable,zhang2024reinforcement,chen2022human}.
For discounted infinite horizon MDP, \citep{du2024exploration} samples trajectories with a specified termination probability, whereas \citep{dai2023safe} samples trajectories of fixed length.  How to effectively incorporate trajectories of arbitrary length remains an open problem.

In this paper, we investigate safe RLHF in the discounted infinite horizon constrained Markov
decision process (CMDP) setting without assuming a fixed horizon reward model. We propose two algorithms for safe reinforcement learning from human feedback (RLHF) with global convergence guarantees. Both algorithms employ the same two-stage human annotation strategy as the empirical work of \citep{dai2023safe} in each policy gradient iteration, but do not rely on a reward reference.

\textbf{Summary of contribution:}

The main contributions of this work are two algorithms for safe reinforcement learning from human feedback (RLHF). Both algorithms establish global convergence without relying on reward references or additional structural assumptions.

The proposed two algorithms are both built on the primal--dual algorithms of \citet{ding2020natural} and \citet{ding2022ACC}, but leverage human feedback in different ways: the first uses human feedback to estimate the advantage function for the primal update, while the second, following the inverse-preference idea of \citet{zhang2024zeroth}, adopts a zeroth-order approach to approximate gradients directly from human feedback. Our methods for both advantage estimation and gradient approximation are inspired by \citet{zhang2024zeroth}, who introduced an approach for estimating value differences through the inverse of preference functions and established convergence to stationary points in episodic MDPs. In contrast, our framework establishes global convergence in infinite-horizon discounted MDPs without requiring additional structural assumptions.

 To the best of our knowledge, this is the first work to investigate RLHF in discounted infinite horizon CMDP. The proposed algorithms allow human evaluators to compare trajectories of varying lengths across policy gradient iterations, whereas most existing studies assume fixed-length trajectories. Moreover, the global convergence results are established without assuming linear reward parameterization, a common requirement in much of the RLHF literature \citep{zhan2023provable,du2024exploration}.

\section{Preliminary}

\textbf{Infinite Horizon Discounted CMDP: }

An infinite-horizon discounted constrained Markov decision process (CMDP) is defined by the tuple
\[
\mathrm{CMDP}(S,A,P,r,g,b,\gamma,\rho),
\]
where \(S\) is a finite state space, \(A\) is a finite action space, \(P(s' \mid s,a)\) denotes the transition probability from state \(s\) to state \(s'\) under action \(a\), and \(\rho\) is the initial state distribution. In our RLHF formulation, \(r: S \times A \to [0,1]\) is the per-step reward function measuring helpfulness, and \(g: S \times A \to [0,1]\) is the per-step utility function measuring harmlessness or safety. The constant \(b\) specifies the minimum acceptable level of discounted harmlessness that a policy must achieve, and \(\gamma \in [0,1]\) is the discount factor.

A  stochastic policy is a mapping $\pi: S \to \Delta_A$, assigning each state a probability distribution over actions, so that at time $t$ the agent selects $a_t \sim \pi(\cdot \mid s_t)$. Let $\Pi$ denote the set of all such policies. When a policy $\pi \in \Pi$ is combined with the initial distribution $\rho$, it induces a distribution over infinite trajectories
\(
\tau = \{(s_t, a_t)\}_{t=0}^\infty,
\)
where the process evolves according to $s_0 \sim \rho$, $a_t \sim \pi(\cdot \mid s_t)$, and $s_{t+1} \sim P(\cdot \mid s_t,a_t)$ for all $t \ge 0$.

The discounted visitation distribution $d^{\pi}_{s_0}$ of a policy $\pi$ with 
initial state $s_0$ and its expectation with respect to an initial state 
distribution $\rho$ are defined as
\begin{equation}
d^{\pi}_{s_0}(s) 
\;=\; (1-\gamma) \sum_{t=0}^{\infty} \gamma^t 
\, P^{\pi}(s_t = s \mid s_0),
\qquad
d^{\pi}_{\rho}(s) 
\;=\; \mathbb{E}_{s_0 \sim \rho} \!\left[ d^{\pi}_{s_0}(s) \right],
\label{eq:discounted-visitation}
\end{equation}
where $P^{\pi}(s_t = s \mid s_0)$ denotes the probability that the process 
occupies state $s$ at time $t$ when following policy $\pi$ starting from 
state $s_0$.

\textbf{Finite Horizon Trajectory Reward and Utility:}
Given a trajectory of length \(H\), \(\tau := \{(s_t, a_t)\}_{t=0}^{H}\), we define the cumulative reward and utility as the discounted sums of the corresponding per-step quantities. Specifically, \(R_r^H(\tau) = \sum_{t=0}^{H} \gamma^t r(s_t, a_t)\), \(R_g^H(\tau) = \sum_{t=0}^{H} \gamma^t g(s_t, a_t)\), where \(r(s_t, a_t)\) and \(g(s_t, a_t)\) denote the per-step reward and utility at state--action pair \((s_t, a_t)\).

Restricting the trajectory length to \(H\), starting from $s\in \mathcal{S}$ and evolve under \(\pi\), the finite horizon value functions of reward and utility, denoted \(V_r^{\pi,H}(s)\) and \(V_g^{\pi,H}(s)\), are defined as the expected cumulative reward and utility of trajectories: \(V_r^{\pi,H}(s):=\mathbb{E}\!\left[\sum_{t=0}^H \gamma^t r(s_t,a_t)\,\middle|\,\pi,\,s_0=s\right]\), \(V_g^{\pi,H}(s):=\mathbb{E}\!\left[\sum_{t=0}^H \gamma^t g(s_t,a_t)\,\middle|\,\pi,\,s_0=s\right]\).
If the initial state is drawn from the distribution $\rho$, i.e., $s_0\sim\rho$, we define: \(V_r^{\pi,H}(\rho):=\mathbb{E}_{s_0\sim \rho}\!\left[\mathbb{E}\!\left[\sum_{t=0}^H \gamma^t r(s_t,a_t)\,\middle|\,\pi,\,s_0\right]\right]\), \(V_g^{\pi,H}(\rho):=\mathbb{E}_{s_0\sim \rho}\!\left[\mathbb{E}\!\left[\sum_{t=0}^H \gamma^t g(s_t,a_t)\,\middle|\,\pi,\,s_0\right]\right]\).

Let \(G(H)=\frac{1-\gamma^{H+1}}{1-\gamma}\). Since \(r(s,a)\in[0,1]\) and \(g(s,a)\in[0,1]\), it follows that \(R_r^H(\tau)\), \(R_g^H(\tau)\), \(V_r^{\pi,H}(\rho)\), and \(V_g^{\pi,H}(\rho)\) are all bounded within \([0,G(H)]\).

\textbf{Link function:} In discrete choice theory, a link function translates differences in latent utility between alternatives into observable choice probabilities. Within the random utility framework, each option’s utility is modeled as a systematic component plus an unobserved random error, and the assumed distribution of this error determines the functional form of the link. If the error is normally distributed, the probability of choosing an option is given by the normal cumulative distribution, producing the probit model. If the error follows a Gumbel (Type I extreme value) distribution, the resulting link is logistic, yielding the logit model, closely related to the classical Bradley–Terry paired comparison model \cite{bradley1952rank}. More generally, alternative error distributions (e.g., Weibull, Cauchy, or generalized extreme value families) imply different link functions, each mapping systematic utility differences into choice probabilities \cite{ben1985discrete}, \cite{koppelman2006self}, \cite{wang2011generalized}.

When a human evaluator compares two trajectories, \(\tau^0\) and \(\tau^1\), with associated latent rewards \(R(\tau^0)\) and \(R(\tau^1)\), the most commonly used model in the RLHF literature for the probability that \(\tau^1\) is preferred to \(\tau^0\) is the Bradley--Terry model:
\[
\mathbb{P}(o = 1) = \mathbb{P}(\tau^1 \succ \tau^0)
= \sigma\!\big(R(\tau^1) - R(\tau^0)\big)
= \frac{\exp(R(\tau^1))}{\exp(R(\tau^0)) + \exp(R(\tau^1))},
\]

\textbf{Human Feedback:}
We rely on human evaluators to provide feedback rather than directly computing trajectory-level rewards. Human evaluators will be provided pairs of trajectories or individual trajectories. 

For a pair of trajectories,  human evaluators will compare them to identify which trajectory is more {\it helpful} (i.e., associated with a higher reward value) and which is more {\it harmless} (i.e., associated with a higher utility value).
In each episode, the agent can choose two trajectories $\tau_0$ and $\tau_1$ to query human preference about helpfulness and harmlessness and get two one-bit feedback $o_r \in \{0,1\}$,$o_g \in \{0,1\}$. We assume the preference $o_r$ and $o_g$ are generated according to a known preference model where the probability of the outcome between two trajectories is determined by the difference in their rewards and harmlessness.

The model uses a link function $\sigma_r : \mathbb{R} \to [0,1]$ to map the differences of rewards to actual probabilities, i.e.,
\begin{equation*}
\mathbb{P}(o_r=1)=\mathbb{P}(R_r^H(\tau_1) > R_r^H(\tau_0)) = \sigma_r\big(R_r^H(\tau_1) - R_r^H(\tau_0)\big),
\end{equation*}
where \( R_r^H(\tau_1) > R_r^H(\tau_0) \) denotes the event in which the human feedback indicates a preference for \(\tau_1\) over \(\tau_0\) in terms of helpfulness, i.e., $o_r=1$.

The model uses another link function $\sigma_g : \mathbb{R} \to [0,1]$ to map the differences of harmlessness to actual probabilities, i.e.,
\begin{equation*}
\mathbb{P}(o_g=1)=\mathbb{P}(R_g^H(\tau_1) > R_g^H(\tau_0)) = \sigma_g\big(R_g^H(\tau_1) - R_g^H(\tau_0)\big),
\end{equation*}
where \( R_g^H(\tau_1) > R_g^H(\tau_0) \) denotes the event in which the human feedback indicates a preference for \(\tau_1\) over \(\tau_0\) in terms of harmlessness, i.e., $o_g=1$.

Link functions are typically used to quantify human preferences between two trajectories. However, in this paper, we further require human evaluators to directly assess a given trajectory harmful or not,  which is important for the dual update of our algorithms. 
Therefore, we introduce the third link function $\sigma_c : \mathbb{R} \to [0,1]$ to estimate the absolute utility of a single trajectory, as opposed to utility differences across a pair of trajectories. This type of evaluation is reasonable as it closely resembles the setting of Reinforcement Learning with Once-per-Episode Feedback \citep{chatterji2021theory}, where the learner receives binary feedback at the end of an episode.  Equivalently, it can be viewed as if there exists a virtual trajectory $\tau_v$ that lies on the boundary of the constraint, i.e., $R_g^H(\tau_v) - b = 0$, and human evaluators compare each observed trajectory against this virtual trajectory.

In each episode, the agent chooses one trajectory $\tau_0$ to query human preference: one-bit feedback $o_c \in \{0,1\}$. We assume the preference $o_c$ is generated according to a known preference model where the probability of the outcome is determined by the harmlessness value. 

The model uses a third link function $\sigma_c : \mathbb{R} \to [0,1]$ to map the harmlessness value to actual probabilities, i.e.,
\begin{equation*}
\mathbb{P}(o_c=1)= \mathbb{P}(R_g^H(\tau_0) - b>0) = \sigma_c\big(R_g^H(\tau_0) - b\big),
\end{equation*}
where $R_g^H(\tau_0) - b>0$ is the event that the human feedback thinks the trajectory is harmless, i.e., $o_c=1$.

Similar to \citep{zhang2024zeroth}, instead of using a specific link function, we consider a general class of link functions satisfying the following condition:

\begin{assumption}
\label{assum:link}
The link functions \(\sigma_c(\cdot)\), \(\sigma_r(\cdot)\), and \(\sigma_g(\cdot)\) in the preference model are assumed to be continuous, bounded in \([0,1]\), and strictly monotonically increasing on \( [-G(1/(1-\gamma)),\, G(1/(1-\gamma))] \), with
\[
\sigma_c(0)=\sigma_r(0)=\sigma_g(0)=\tfrac{1}{2}.
\]
We further assume that the link functions are Lipschitz continuous on \( [-G(1/(1-\gamma)),\, G(1/(1-\gamma))] \), and that their inverses are Lipschitz continuous on
\(
[\sigma_c(-G(1/(1-\gamma))),\, \sigma_c(G(1/(1-\gamma)))],\quad
[\sigma_r(-G(1/(1-\gamma))),\, \sigma_r(G(1/(1-\gamma)))],\quad
[\sigma_g(-G(1/(1-\gamma))),\, \sigma_g(G(1/(1-\gamma)))] ,
\)
respectively. For simplicity, we additionally assume that the inverse link functions \(\sigma_c^{-1}(\cdot)\), \(\sigma_r^{-1}(\cdot)\), and \(\sigma_g^{-1}(\cdot)\) share a common Lipschitz constant \(L\) on their respective domains.
\end{assumption}

For a scalar $a$, we define
\(
\operatorname{clip}[a,\sigma(-G(H)),\sigma(G(H))] := \min \{ \max(a, \sigma(-G(H))), \, \sigma(G(H)) \}.
\)

In this paper, we consider two types of complete parametrizations:

\textbf{Direct Parametrization:}
A direct way to parameterize a policy is by specifying its probability distribution,
\[
\pi_\theta^{direct}(a \mid s) = \theta_{s,a}, \quad \forall \, \theta \in \Delta_A^{|S|},
\]
where $\theta_s \in \Delta_A$ for every $s \in S$. That is, for each $s$, $\theta_{s,a} \geq 0$ for all $a \in A$, and $\sum_{a \in A} \theta_{s,a} = 1$. 
This class of policies is complete since it can represent any stochastic policy directly.

\textbf{Softmax Parametrization:}
A natural class of policies is parametrized by the softmax function,
\begin{equation}
\pi_{\theta}^{softmax}(a \mid s) = 
\frac{\exp(\theta_{s,a})}{\sum_{a' \in A} \exp(\theta_{s,a'})}, 
\quad \text{for all } \theta \in \mathbb{R}^{|S||A|}.
\end{equation}

\textbf{Problem Formulation:}

Let the symbol $\diamond$ represent $r$ or $g$.
 When the parameterized policy \(\pi_{\boldsymbol{\theta}}\) is clear from context, we write $\boldsymbol{\theta}$ to denote $\pi_{\boldsymbol{\theta}}$ and
\(
  V_\diamond^{\boldsymbol{\theta}}(\rho)\) to denote \( V_\diamond^{\pi_{\boldsymbol{\theta}}}(\rho).
\)

Given a policy $\pi_{\boldsymbol{\theta}}$, starting from any state $s\in \mathcal{S}$, the value functions $V_r^{\boldsymbol{\theta}}(s), V_g^{\boldsymbol{\theta}}(s)$ associated with the reward function $r$ and the utility function $g$ are defined as the expected total discounted reward and utility 
\(
V_r^{\boldsymbol{\theta}}(s)\;:=\mathbb{E}\!\Bigl[\sum_{t=0}^\infty \gamma^t\,r(s_t,a_t)\,\bigm|\,\pi_{\boldsymbol{\theta}},\,s_0=s\Bigr],
V_g^{\boldsymbol{\theta}}(s)\;:=\mathbb{E}\!\Bigl[\sum_{t=0}^\infty \gamma^t\,g(s_t,a_t)\,\bigm|\,\pi_{\boldsymbol{\theta}},\,s_0=s\Bigr].
\)

If the initial state is drawn from the distribution $\rho$, i.e., $s_0 \sim \rho$, the value functions are defined as:
\(
V_r^{\boldsymbol{\theta}}(\rho)\;:=\mathbb{E}_{s_0\sim\rho}\left[\mathbb{E}\!\Bigl[\sum_{t=0}^\infty \gamma^t\,r(s_t,a_t)\,\bigm|\,\pi_{\boldsymbol{\theta}},\,s_0\Bigr]\right],
V_g^{\boldsymbol{\theta}}(\rho)\;:=\mathbb{E}_{s_0\sim\rho}\left[\;\mathbb{E}\!\Bigl[\sum_{t=0}^\infty \gamma^t\,g(s_t,a_t)\,\bigm|\,\pi_{\boldsymbol{\theta}},\,s_0\Bigr]\right].
\)

The goal of our RL problem is to use only human feedback to identify the parameter 
$\boldsymbol{\theta}$ that maximizes the expected infinite horizon reward value, subject to a constraint on the expected infinite horizon utility value:
\begin{equation}\label{eq:primal}
\begin{aligned}
  \max_{{\boldsymbol{\theta}}\in{\Theta}}\quad & V_r^{{\boldsymbol{\theta}}}(\rho),\\
  \text{subject to}\quad & V_g^{{\boldsymbol{\theta}}}(\rho)\ge b.
\end{aligned}
\end{equation}
We define the Lagrangian as
\[
V_{L}^{\boldsymbol{\theta},\lambda}(\rho) := V_r^{\boldsymbol{\theta}}(\rho) + \lambda \big( V_g^{\boldsymbol{\theta}}(\rho) - b \big), 
\] where $\lambda$ is the Lagrange multiplier.  

\begin{assumption}[Slater condition]
\label{assum:slater}
There exists $\xi>0$ and $\bar\pi\in\Pi$ such that
\[
  V^{\bar\pi}_{g}(\rho) \;-\; b \;\ge\; \xi.
\]
\end{assumption}

As both parametrizations considered in this paper are complete, with Theorem 1 of \citep{paternain2019constrained} and the Slater condition, we can guarantee strong duality for the problem (\ref{eq:primal}).
Then problem (\ref{eq:primal}) can be reformulated as the following max–min problem:
\[
\operatorname*{maximize}_{\boldsymbol{\theta} \in \Theta} \; 
\operatorname*{minimize}_{\lambda \geq 0} 
\; V_{L}^{\boldsymbol{\theta},\lambda}(\rho),
\]
where $\boldsymbol{\theta}$ serves as the primal variable and $\lambda \geq 0$ is the Lagrange multiplier.

The associated dual function is defined as 
\[
V_{D}^{\lambda}(\rho) := \max_{\boldsymbol{\theta}} V_{L}^{\boldsymbol{\theta},\lambda}(\rho).
\]
Since the two types of parametrization considered in this paper are complete,
the optimal solution $\pi^\star$ to problem (\ref{eq:primal}) can be denoted by 
$\boldsymbol{\theta}^\star$ , and let the optimal dual variable be $\lambda^\star = \arg\min_{\lambda \ge 0} V_D^\lambda(\rho)$. We use the shorthand notation \( V_r^{\pi^\star}(\rho)=V_r^{\boldsymbol{\theta}^\star}(\rho) = V_r^\star(\rho) \) and \( V_D^{\lambda^\star}(\rho) = V_D^\star(\rho) \) whenever the meaning is clear from the context. The following lemma from  \citet{ding2020natural} holds true. 
\begin{Lemma}
\label{Lemma:strong_duality}
Let Assumption~\ref{assum:slater} hold. Then
 $V_r^{\star}(\rho) = V_D^{\star}(\rho);$ $0 \le \lambda^\star \le \frac{V_r^{\star}(\rho)  - V_r^{\bar{\pi}}(\rho)}{\xi}.$
\end{Lemma}

Let $\Lambda = [0, \bar{\sigma}_\lambda]$ denote the set onto which the dual variable is projected in both algorithms proposed in this work. Because of the boundedness of the value function $V_r$, it follows that if  
\(
\bar{\sigma}_\lambda = \frac{2}{(1-\gamma)\xi},
\)
then $\lambda^\star \in \Lambda$.

\section{NPGPD-HF: Natural Policy Gradient Primal-Dual Method from Human Feedback}

Natural policy gradient has a fast convergence rate that is independent of the sizes of state space and action space. We base our first algorithm on the Natural Policy Gradient Primal-Dual Method \citep{ding2020natural} as its convergence is independent of the size of the state-action space, extending it to incorporate finite horizon sampling and human feedback.
The algorithm in \citep{ding2020natural} takes the form:
\begin{equation}
\label{prime1}
\theta^{(t+1)} = \theta^{(t)} + \eta_{1} F_{\rho}(\theta^{(t)})^{\dagger} \cdot 
\nabla_{\theta} V^{\theta^{(t)}, \lambda^{(t)}}_{L}(\rho), 
\end{equation}
\begin{equation}
\label{dual1}
\lambda^{(t+1)} = \mathcal{P}_{\Lambda}\!\left(\lambda^{(t)} - \eta_{2}\big(V^{\theta^{(t)}}_{g}(\rho) - b\big)\right)
\end{equation}
where
\begin{equation}
F_{\rho}(\theta) \;:=\;
\mathbb{E}_{s \sim d_{\rho}^{\pi_\theta}} 
\mathbb{E}_{a \sim \pi_\theta(\cdot \mid s)}
\left[
\nabla_\theta \log \pi_\theta(a \mid s)\,
\big(\nabla_\theta \log \pi_\theta(a \mid s)\big)^\top
\right]
\end{equation}
and $\dagger$ denotes the Moore--Penrose inverse. 

We notice that the primal update is the result of solving the following optimization problem using the KL divergence as the proximal term \citep{kakade2001natural}.
\[
\max_{\boldsymbol{\theta}^{(t+1)}} \; \underbrace{V^{\boldsymbol{\theta}^{(t)}}(\rho) 
+ \nabla_{\theta} V^{\boldsymbol{\theta}^{(t)}}(\rho)(\boldsymbol{\theta}^{(t+1)}-\boldsymbol{\theta}^{(t)})}_{\approx V^{\boldsymbol{\theta}^{(t+1)}}(\rho)}
- \beta \; \mathbb{E}_{x \sim \rho_{w}}
   \Big[ \mathrm{KL}\!\left( \pi_{\boldsymbol{\theta}^{(t+1)}}(\cdot \mid x) \,\|\, \pi_{\boldsymbol{\theta}^{(t)}}(\cdot \mid x) \right) \Big].
\]
With the following lemma we can get the policy update in closed form under softmax parametrization :
\begin{Lemma}[\cite{ding2020natural}]
\label{Lemma:derive advantage}
Let $A_{L}^{(t)}(s,a) := A_{r}^{(t)}(s,a) + \lambda^{(t)} A_{g}^{(t)}(s,a)$ 
and 
\(
Z^{(t)}(s) := \sum_{a \in A} \pi^{(t)}(a \mid s) 
\exp\!\left(\tfrac{\eta_{1}}{1-\gamma} A_{L}^{(t)}(s,a)\right).
\)
Under Softmax Parametrization, the primal-dual update (\ref{prime1}) and (\ref{dual1}) is equivalent to
\[
\theta_{s,a}^{(t+1)} 
= \theta_{s,a}^{(t)} + \frac{\eta_{1}}{1-\gamma} A_{L}^{(t)}(s,a)
\quad \text{or} \quad
\pi^{(t+1)}(a \mid s) 
= \pi^{(t)}(a \mid s) 
\frac{\exp\!\left(\tfrac{\eta_{1}}{1-\gamma} A_{L}^{(t)}(s,a)\right)}{Z^{(t)}(s)}
\]
\[
\text{and} \quad 
\lambda^{(t+1)} = \mathcal{P}_{\Lambda}\!\left(\lambda^{(t)} - \eta_{2}\big(V_{g}^{(t)}(\rho) - b\big)\right).
\]
\end{Lemma}
\textbf{Relation to Direct Preference Optimization (DPO)}
In \citep{rafailov2023direct}, the DPO take the form of \begin{equation}
\pi_{r}(y \mid x) 
= \frac{1}{Z(x)} \, \pi_{\text{ref}}(y \mid x) 
\exp\!\left(\tfrac{1}{\beta} r(x,y)\right),
\end{equation}

where \(
Z(x) = \sum_{y} \pi_{\text{ref}}(y \mid x) 
\exp\!\left(\tfrac{1}{\beta} \left(r(x,y)\right)\right)
\;\)
and $r(x,a)$ is the reward in state $x$ with action $a$, $\pi_{ref}
$ is a reference policy and $\pi_r$ is the optimal policy, which is the result of the following optimization problem. 
\[
\max_{\pi_\theta} \; 
\mathbb{E}_{x \sim \mathcal{D},\, y \sim \pi_\theta(y \mid x)}
\left[ r_\phi(x,y) \right]
- \beta \, \mathbb{D}_{\mathrm{KL}}\!\left[
\pi_\theta(y \mid x) \,\|\, \pi_{\mathrm{ref}}(y \mid x)
\right],
\]

We observe that the primal update (\ref{prime1}) takes a form similar to DPO. The key difference lies in the formulation: \citet{rafailov2023direct} model the problem as a bandit problem, where an entire episode is treated as a single action with a single associated reward, resulting in sparse feedback. In contrast, our approach employs a first-order approximation of the value function of infinite horizon MDP together with a softmax parameterization.

\begin{algorithm}[htbp]
\caption{Natural Policy Gradient Primal-Dual Method from Human Feedback (NPGPD-HF)}
\label{alg:NPGPD-HF}
\begin{algorithmic}[1]
\Require Learning rates $\eta_1, \eta_2$, rounds $N$, evaluators $M$, maximum iterations $T$  and simulation access to $\mathrm{CMDP}(\mathcal{S},\mathcal{A},P,r,g,b,\gamma,\rho)$
\State Initialize $\boldsymbol{\theta}^{(0)}=\boldsymbol{0}$, $\lambda^{(0)}=0$.
\For{$t = 0:T$}

  \For{$n = 1:N$}
    \State  Starting from each $s\in\mathcal{S}$, sample a trajectory $\tau_{n}(s)$ by executing policy $\pi_{\boldsymbol{\theta}^{(t)}}$\\
    \hspace{2.9em}with length $H$;
    
    \State  Starting from each $(s,a)\in\mathcal{S}\times\mathcal{A}$, sample a trajectory $\tau_{n}(s,a)$ by executing policy $\pi_{\boldsymbol{\theta}^{(t)}}$ \\
    \hspace{2.9em}with length $H$; 

    \State For each $(s,a)\in\mathcal{S}\times\mathcal{A}$, query $M$ human  evaluators about helpfulness\\
    \hspace{2.9em}with $(\tau_{n}(s,a), \tau_{n}(s))$ and obtain feedback $[o_{r,n,1}(s,a), \cdots, o_{r,n,M}(s,a)]$;
    
    \State For each $(s,a)\in\mathcal{S}\times\mathcal{A}$, query $M$ human evaluators with $(\tau_{n}(s,a), \tau_{n}(s))$ about\\
    \hspace{2.9em} harmlessness and obtain feedback $[o_{g,n,1}(s,a), \cdots, o_{g,n,M}(s,a)];$
    
    \State Estimate preference probability:
    \[
        p^{(t)}_{r,n}(s,a) = \text{clip}\left( \sum_{m=1}^M \frac{o_{r,n,m}(s,a)}{M},\sigma_r(-G(H)),\sigma_r(G(H)) \right);
    \]
    \[
        p^{(t)}_{g,n}(s,a) = \text{clip}\left(\sum_{m=1}^M \frac{o_{g,n,m}(s,a)}{M} ,\sigma_g(-G(H)),\sigma_g(G(H)) \right);
    \]
    \State Draw $s_{n,0}\sim\rho$, sample a trajectory $\tau_{n}(s_{n,0})$ by executing policy $\pi_{\boldsymbol{\theta}^{(t)}}$ with length $H$ \\
    \hspace{2.9em} and initial state $s_{n,0}$, 
    
    \State Query $M$ human evaluators with $\tau_{n}(s_{n,0})$ to determine whether the response\\
    \hspace{2.9em} is harmless and obtain feedback $[o_{c,n,1}(\rho), \cdots, o_{c,n,M}(\rho)];$
     
    \State Estimate absolute probability:
        \[
        p^{(t)}_{c,n}(\rho) = \text{clip}\left(  \sum_{m=1}^M \frac{o_{c,n,m}(\rho)}{M}  ,\sigma_c(-G(H)),\sigma_c(G(H)) \right);
        \]
        
  \EndFor
    \State Estimate the advantage function of reward and utility: 
    \[
    \hat{A}^{(t)}_r(s,a)= \frac{1}{N} \sum_{n=1}^N \sigma_r^{-1}(p^{(t)}_{r,n}(s,a)) ;\quad
    \hat{A}^{(t)}_g(s,a)= \frac{1}{N} \sum_{n=1}^N \sigma_g^{-1}(p^{(t)}_{g,n}(s,a)) ;
    \]
     \State Estimate the utility value $V_g(\rho)-b$:
    \[
    \hat{h}^{(t)}_c(\rho) =  \frac{1}{N} \sum_{n=1}^N \sigma_c^{-1}(p^{(t)}_{c,n}(\rho));
    \]
  \State For each $(s,a)\in\mathcal{S}\times\mathcal{A}$, let $\hat A^{(t)}_{L}(s,a) = \hat A^{(t)}_{r}(s,a) +\lambda^{(t)} \hat A^{(t)}_{g}(s,a);$
  \State For each $(s,a)\in\mathcal{S}\times\mathcal{A}$, update $\theta^{(t+1)}_{s,a} \;=\; \theta^{(t)}_{s,a} + \frac{\eta_1}{1-\gamma}\,\hat A^{(t)}_{L}(s,a);$
  \State Update $\lambda^{(t+1)} \;=\; \mathcal{P}_{\Lambda}( \lambda^{(t)} - \eta_2 ( \hat{h}_c^{(t)}(\rho)) ).$
\EndFor
\end{algorithmic}
\end{algorithm}
\subsection{Algorithm}
In each policy gradient iteration, Algorithm \ref{alg:NPGPD-HF} consists of the following steps:

\textbf{Step 1: Collecting trajectories initialized from each state and state–action pairs by interacting with the environment.} From each $s \in \mathcal{S}$ and each $(s,a) \in \mathcal{S}\times\mathcal{A}$, 
sample $N$ trajectories $\tau_{n}(s)$ and $\tau_{n}(s,a)$ of length $H$ 
under policy $\pi_{\boldsymbol{\theta}^{(t)}}$ (line 4--7).

\textbf{Step 2: Collect pairwise preference feedback from human evaluators.}
For each $(s,a)\in\mathcal{S}\times\mathcal{A}$, query $M$ human evaluators with 
$(\tau_{n}(s,a), \tau_{n}(s))$ about helpfulness and harmlessness (lines 8--11), and estimate 
the probabilities that $\tau_{n}(s,a)$ is more helpful than $\tau_{n}(s)$, denoted 
$p^{(t)}_{r,n}(s,a)$, and that $\tau_{n}(s,a)$ is more harmless than $\tau_{n}(s)$, denoted 
$p^{(t)}_{g,n}(s,a)$ (line 12).

\textbf{Step 3: Collecting trajectories initialized from a distribution over states by interacting with the environment.}
Sample $N$ trajectories $\tau_{n}(s_{n,0})$ of length $H$ 
by executing policy $\pi_{\boldsymbol{\theta}^{(t)}}$ with initial state $s_{n,0}$ independently drawn from distribution $\rho$  (line 13--14).

\textbf{Step 4: Collecting absolute feedback from human evaluators.}
Query $M$ human evaluators with $\tau_{n}(s_0)$ to assess harmlessness (line 15--16) 
and estimate the probability that $\tau_{n}(s_0)$ is harmless, denoted by $p^{(t)}_{c,n}(\rho)$.

\textbf{Step 5: Estimate Advantage functions and utility Value function.}
For each state--action pair $(s,a)\in\mathcal{S}\times\mathcal{A}$, use the $N$ estimates 
$p^{(t)}_{r,n}(s,a)$ and $p^{(t)}_{g,n}(s,a)$, together with the inverse link functions 
$\sigma_r^{-1}(\cdot)$ and $\sigma_g^{-1}(\cdot)$, to estimate the reward and utility advantage 
functions $\hat A_r^{(t)}(s,a)$ and $\hat A_g^{(t)}(s,a)$ (line 19).\\
Use the $N$ estimates 
$p^{(t)}_{c,n}(\rho)$, together with the inverse link functions 
$\sigma_c^{-1}(\cdot)$,  to estimate the value function of utility, denoted by $\hat h_c^{(t)}(\rho)$ (line 20).

\textbf{Step 6: Update the primal and the dual.} 
For each state--action pair $(s,a)\in\mathcal{S}\times\mathcal{A}$, use the estimated reward and utility advantage functions $\hat A_r^{(t)}(s,a)$ and $\hat A_g^{(t)}(s,a)$, together with the Lagrange multiplier 
$\lambda^{(t)}$ to update the parameterized policy $\theta_{s,a}^{(t)}$ (lines 21--22), and use the estimated utility 
value function $\hat h_c^{(t)}(\rho)$ to update $\lambda^{(t)}$ (line 23);

\subsection{Analysis}
\begin{theorem}
\label{thm1}
Suppose Assumptions \ref{assum:link} and \ref{assum:slater} hold, and let $\bar{\sigma}_\lambda = \tfrac{2}{(1-\gamma)\xi}$.  
Fix $T > 0$, $\rho \in \Delta_S$, and initialize $\theta^{(0)} = 0$ and $\lambda^{(0)} = 0$.  
If we choose $\eta_1 = 2 \log |A|$ and $\eta_2 = (1 - \gamma)/\sqrt{T}$, then under the softmax parameterization,  
the iterates $\boldsymbol{\theta}^{(t)}$ generated by NPGPD-HF satisfy
\end{theorem}

\begin{equation*}
\frac{1}{T} \sum_{t=0}^{T-1} 
\mathbb{E}\!\left[ V_{r}^{\star}(\rho) - V_{r}^{\boldsymbol{\theta}^{(t)}}(\rho) \right]
\;\leq\; 
\mathcal{O}\!\left( \frac{1}{\sqrt{T}} \;+\; L \sqrt{\frac{ \log M}{M}} \;+\;  \frac{\gamma^H}{1-\gamma} \right),
\end{equation*}
and
\begin{equation*}
\mathbb{E}\!\left[ b - \frac{1}{T} \sum_{t=0}^{T-1} V_{g}^{\boldsymbol{\theta}^{(t)}}(\rho) \right]_+
\;\leq\; 
\mathcal{O}\!\left( \frac{1}{\sqrt{T}} \;+\; L \sqrt{\frac{ \log M}{M}} \;+\;  \frac{\gamma^H}{1-\gamma} \right).
\end{equation*}
The complete proof is presented in the appendix \ref{proof:thm1}.

\textbf{Insights:}  
The convergence rate can be decomposed into three main sources of error: the stochastic optimization rate, the bias introduced by human feedback, and the approximation error from truncating the infinite horizon:  

\[
\quad\quad\underbrace{\frac{1}{\sqrt{T}}}_{\text{Policy gradient}}
\quad\quad\;+\;\quad\quad
\underbrace{L\sqrt{\frac{\log M}{M}} }_{\text{Bias of Human Feedback}}
\quad\quad\;+\;
\underbrace{\frac{\gamma^H}{1-\gamma}}_{\text{Bias of finite horizon approximation}} 
\]

The first term, $\tfrac{1}{\sqrt{T}}$, reflects the convergence rate of natural  policy gradient methods of \citep{ding2022convergence}.
The second term represents the bias of human feedback:  
$L\sqrt{\log M / M}$ captures the statistical error when approximating population-level human preference probabilities using $M$ feedback samples, and decreases with larger $M$. The final term, $\tfrac{\gamma^H}{1-\gamma}$, is the bias of finite horizon approximation,  
arising because infinite horizon returns are truncated at horizon length $H$.  
This bias decays exponentially with $H$, scaled by the discount factor $\gamma$.  
\begin{corollary}
Suppose that all the conditions of \ref{thm1} hold, let $H\sim\mathcal{O}(log(T))$,$M\sim \mathcal{O}({T \log T})$, 
We can get that the iterations $\boldsymbol{\theta}^{(t)}$ generated by NPGPD-HF satisfy
\[\frac{1}{T} \sum_{t=0}^{T-1} \mathbb{E}\left[ V_{r}^{\star}(\rho) - V_{r}^{\boldsymbol{\theta}^{(t)}}(\rho)\right] \leq\mathcal{O}(\frac{1}{\sqrt{T}})\]
\end{corollary}
 This corollary suggests that shorter trajectories can initially be sampled and evaluated by fewer human reviewers, while longer trajectories and more reviewers can be incorporated later to achieve greater accuracy.
 
\section{ZPGPD-HF: Zeroth-Order Policy Gradient Primal-Dual Method from Human Feedback}
For NPGPD-HF, estimating the advantage function requires $|S||A|$ human queries per iteration. To allow an arbitrary less number of inner-loop queries, we propose ZPGPD-HF, which builds on the primal–dual method for CMDP developed by \citep{ding2022ACC} and the zeroth-order approach for RLHF introduced by \citep{zhang2024zeroth}.

\subsection{Algorithm}
The main difference between Algorithm~\ref{alg:NPGPD-HF} and Algorithm~\ref{alg:ZPGPD-HF} lies in how the trajectory pairs for comparison are collected.

In each policy gradient iteration, our algorithm consists of the following steps:

\textbf{Step 1: Generate a perturbed policy}
Sample $\boldsymbol{v}^{(t)}$ uniformly from a unit sphere $\mathbb{S}^{d-1} = \left\{ \boldsymbol{v} \in \mathbb{R}^d \,\middle|\, \|\boldsymbol{v}\|_2 = 1 \right\}$. From the current policy $\pi_{\boldsymbol{\theta}^{(t)}}$, it first obtained a perturbed policy $\pi_{\boldsymbol{\theta}^{(t)}+\mu}$ (line 2--3).

\textbf{Step 2: Collect trajectories from the base and perturbed policies.} 
Sample $N$ trajectories of length $H$ under policies $\pi_{\boldsymbol{\theta}^{(t)}}$ and $\pi_{\boldsymbol{\theta}^{(t)}+\mu}$, where initial state of each trajectory is independently drawn from distribution $\rho$ (line 6--7).

\textbf{Step 3: Collect pairwise preference feedback from human evaluators.} 
For each pair $(\tau_{n,1}(s_{n,0}),\tau_{n,0}(s_{n,0}))$, query $M$ independent human evaluators about helpfulness and harmlessness (lines 8--9), estimate 
the probabilities that $(\tau_{n,1}(s_{n,0})$ is more helpful than $\tau_{n,0}(s_{n,0})$, denoted 
$p^{(t)}_{r,n}$, and that $(\tau_{n,1}(s_{n,0})$ is more harmless than $\tau_{n,0}(s_{n,0})$, denoted 
$p^{(t)}_{g,n}$ (line 10).

\textbf{Step 4: Collect absolute feedback from human evaluators} 
Query $M$ human evaluators with $\tau_{n,0}(s_{n,0})$ to assess harmlessness (line 11) 
and estimate the probability that $\tau_{n,0}(s_{n,0})$ is harmless, denoted by $p^{(t)}_{c,n}(\rho)$ (line 12).

\textbf{Step 5: Estimate gradients and utility Value function} 
Use the $N$ estimates 
$p^{(t)}_{r,n}(\rho)$ and $p^{(t)}_{g,n}(\rho)$, together with the inverse link functions 
$\sigma_r^{-1}(\cdot)$ and $\sigma_g^{-1}(\cdot)$, to estimate the gradient of the reward and utility value 
function, denoted by $\hat{h}^{(t)}_c(\rho)$ and $\hat{h}^{(t)}_g(\rho)$ (line 14).\\
Use the $N$ estimates 
$p^{(t)}_{c,n}(\rho)$, together with the inverse link functions 
$\sigma_c^{-1}(\cdot)$,  to estimate value function of utility, denoted by $\hat h_c^{(t)}(\rho)$ (line 15).

\textbf{Step 6: Update the primal and dual.} 
Update the policy parameters $\boldsymbol{\theta}^{(t)}$ using the estimated gradient and learning rate $\eta_1$, and update the Lagrange multiplier $\lambda^{(t)}$ using the estimated value function and learning rate $\eta_2$.

The whole algorithm is provided in appendix \ref{ZPGAP}.

\subsection{Analysis}
\begin{theorem}\label{thm2}
Suppose Assumptions \ref{assum:link} and \ref{assum:slater} hold, and let 
\[
\bar{\sigma}_\lambda = \tfrac{2}{(1-\gamma)\xi}, 
\quad \rho \in \Delta_s, \quad \lambda^{(0)} = 0, \quad d = |S||A|.
\]
Let ${\boldsymbol{\theta}}^{(0)} \in {\boldsymbol{\Theta}}$ satisfy 
\( V_r^{{\boldsymbol{\theta}}^{(0)}}(\rho) \le V_r^{*}(\rho) \).  
For the step sizes
\[
\eta_1 \;=\; \frac{(1-\gamma)^4}{2\,|A|\,(1+2/\xi)},
\qquad
\eta_2 \;=\; \frac{8\,|A|\,|S|\,(1+2/\xi)}{(1-\gamma)^4 \,\sqrt{T}}
\,\Bigl\|\tfrac{d_\rho^{\pi^*}}{\rho}\Bigr\|_\infty^{2},
\]
and
\[
\mu^2 = \max\left\{\tfrac{2\gamma^H}{1-\gamma}, \; L \sqrt{\tfrac{2 \log M}{M}} + \tfrac{2G(H)}{M^2}\right\}.
\]
Then, under direct parameterization, the iterates 
\(\{{\boldsymbol{\theta}}^{(t)}\}_{t \ge 0}\) generated by ZPGPD-HF satisfy
\begin{equation*}
\frac{1}{T}\sum_{t=0}^{T-1}\mathbb{E}\!\left[V_r^{\star}(\rho) - V_r^{\boldsymbol{\theta}^{(t)}}(\rho)\right]
\;\leq\;
\mathcal{O}\!\left(
\frac{d}{T^{1/4}} 
+ d^{3/2} \sqrt{\frac{\log M}{M}} 
+ d^{3/2}\sqrt{\frac{\gamma^H}{1-\gamma}}
\right),
\end{equation*}
and
\begin{equation*}
\mathbb{E}\!\left[ b - \frac{1}{T} \sum_{t=0}^{T-1} V_{g}^{(t)}(\rho) \right]_+
\;\leq\; 
\mathcal{O}\!\left(\frac{d}{T^{1/4}}+d^{3/2}  \left(\frac{ \log M}{M} \right)^{1/4}
+d^{3/2}\sqrt{\frac{\gamma^H}{1-\gamma}}\right).
\end{equation*}
\end{theorem}

The complete proof is presented in the appendix \ref{proof:thm2}.

Similar to NPGPD-HF, the convergence rate of ZPGPD-HF can also be decomposed into the stochastic optimization rate, the bias introduced by human feedback, and the approximation error from truncating the infinite horizon.

Comparing the result of Theorem~\ref{thm2} with Theorem~\ref{thm1}, it can be observed that each term in Theorem~\ref{thm2} is correlated with the sizes of the state and action spaces. In addition, the first term exhibits a slower convergence rate, reflecting the cost of allowing arbitrary inner-loop queries. Furthermore, the bias term arising from finite horizon approximation and human feedback is larger.

\section{Experiment}

To evaluate the convergence of the two proposed algorithms, we construct a CMDP environment with $|\mathcal{S}|=10$ states and $|\mathcal{A}|=4$ actions, so that $d=|\mathcal{S}||\mathcal{A}|=40$. We set $\gamma=0.90$ and choose the initial state distribution $\rho$ to be uniform over $\mathcal{S}$. For each $(s,a)\in\mathcal{S}\times\mathcal{A}$, the transition kernel $P(\cdot\mid s,a)$ is independently drawn from $\mathrm{Dirichlet}(5,\ldots,5)$ to ensure ergodicity. The reward and utility functions, $r(s,a)$ and $g(s,a)$, are independently sampled from $\mathrm{Unif}[0,1]$ and rescaled. We fix the utility threshold at $b=0.55$.

Policy evaluation is performed exactly by solving the linear system $V_r^\pi = (I-\gamma P_\pi)^{-1} r_\pi$, where $P_\pi$ and $r_\pi$ denote the policy-induced transition matrix and expected reward vector, respectively. The optimal reward value $V_r^{\pi^*}$ is computed using value iteration. 

Human feedback is simulated as follows. Pairwise comparisons on helpfulness and harmlessness are generated via Bradley-Terry model: 
\begin{equation*}
\mathbb{P}(o_r=1)=\mathbb{P}(R_r^H(\tau_1) > R_r^H(\tau_0)) = \sigma_r\big(R_r^H(\tau_1) - R_r^H(\tau_0)\big)=
\frac{\exp\bigl(R_r^H(\tau_1)\bigr)}
{\exp\bigl(R_r^H(\tau_0)\bigr)+\exp\bigl(R_r^H(\tau_1)\bigr)},
\end{equation*}
\begin{equation*}
\mathbb{P}(o_g=1)=\mathbb{P}(R_g^H(\tau_1) > R_g^H(\tau_0)) = \sigma_g\big(R_g^H(\tau_1) - R_g^H(\tau_0)\big)=
\frac{\exp\bigl(R_g^H(\tau_1)\bigr)}
{\exp\bigl(R_g^H(\tau_0)\bigr)+\exp\bigl(R_g^H(\tau_1)\bigr)}.
\end{equation*}
Absolute harmlessness feedback is generated through a logistic model as \citep{chatterji2021theory}:
\begin{equation*}
\mathbb{P}(o_c=1)= \mathbb{P}(R_g^H(\tau_0) - b>0) = \sigma_c\big(R_g^H(\tau_0) - b\big)=\frac{\exp\bigl(R_g^H(\tau_0)-b\bigr)}
{1+\exp\bigl(R_g^H(\tau_0)-b\bigr)}.
\end{equation*}
We truncate the horizon to a fixed length of $H = 80$ and vary the number of independent evaluators per query $M\in\{16,64,256\}$.

We compare the two proposed algorithms in Fig. \ref{experi} in terms of both the reward optimality gap and the constraint violation under different numbers of evaluators \(M\in\{16,64,256\}\). For NPGPD-HF, the optimality gap decreases rapidly and then plateaus, with most of the gain obtained by increasing \(M\) from 16 to 64, while \(M=64\) and \(M=256\) yield nearly identical performance, indicating diminishing returns beyond \(M=64\). The constraint violation decreases at a similar rate for all three \(M\) values and the curves largely overlap, suggesting that feasibility is essentially insensitive to \(M\) in this setting. In contrast, ZPGPD-HF's optimality gap converges much more slowly and is more sensitive to the evaluator budget: performance improves substantially as \(M\) increases, with \(M=16\) leading to markedly worse steady-state performance. For the constraint violation of ZPGPD-HF, the curves for \(M=64\) and \(M=256\) are close and reach comparable final levels, suggesting limited additional benefit from increasing \(M\) beyond 64.
\begin{figure}[htbp]
    \centering
    \begin{subfigure}{0.45\textwidth}
        \centering
        \includegraphics[width=\linewidth]{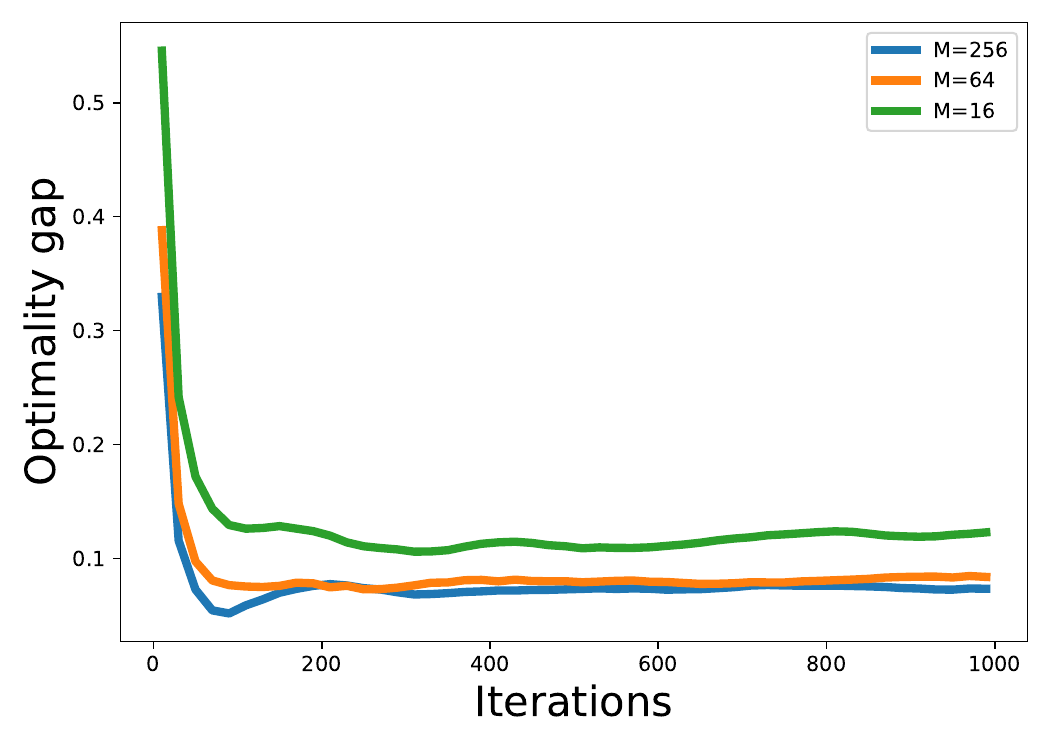}
        \caption{Optimality gap of NPGPD-HF}
        \label{fig:sub1}
    \end{subfigure}
    \hfill
    \begin{subfigure}{0.45\textwidth}
        \centering
        \includegraphics[width=\linewidth]{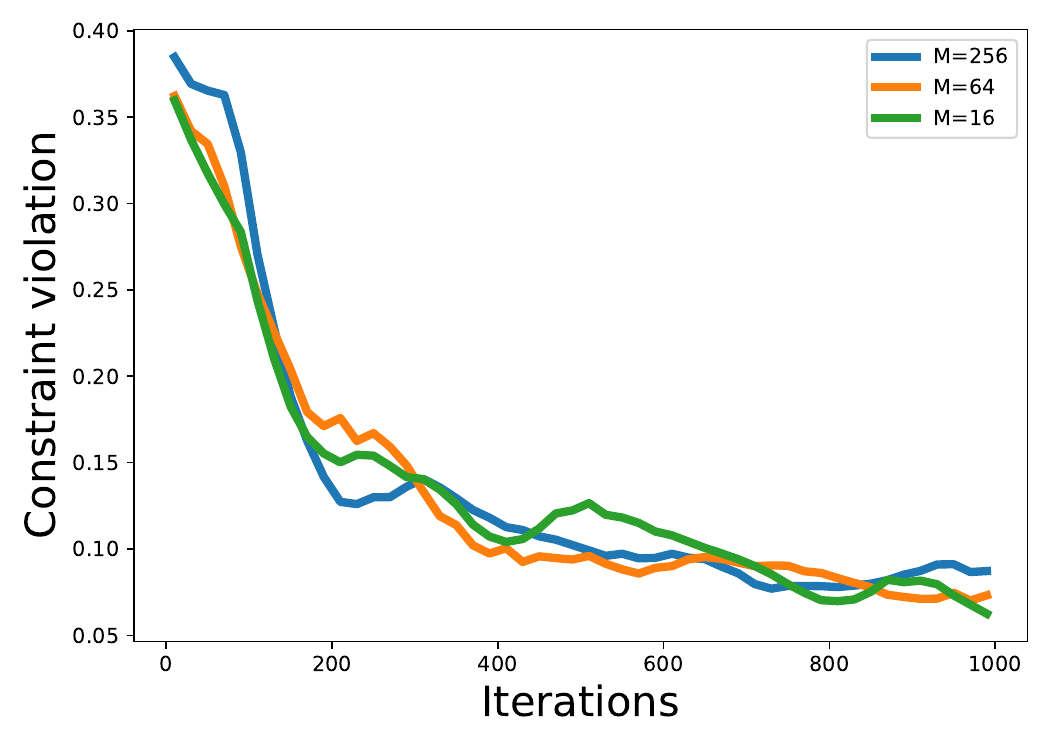}
        \caption{Constraint violation of NPGPD-HF}
        \label{fig:sub2}
    \end{subfigure}
    \vskip\baselineskip
      \begin{subfigure}{0.45\textwidth}
        \centering
        \includegraphics[width=\linewidth]{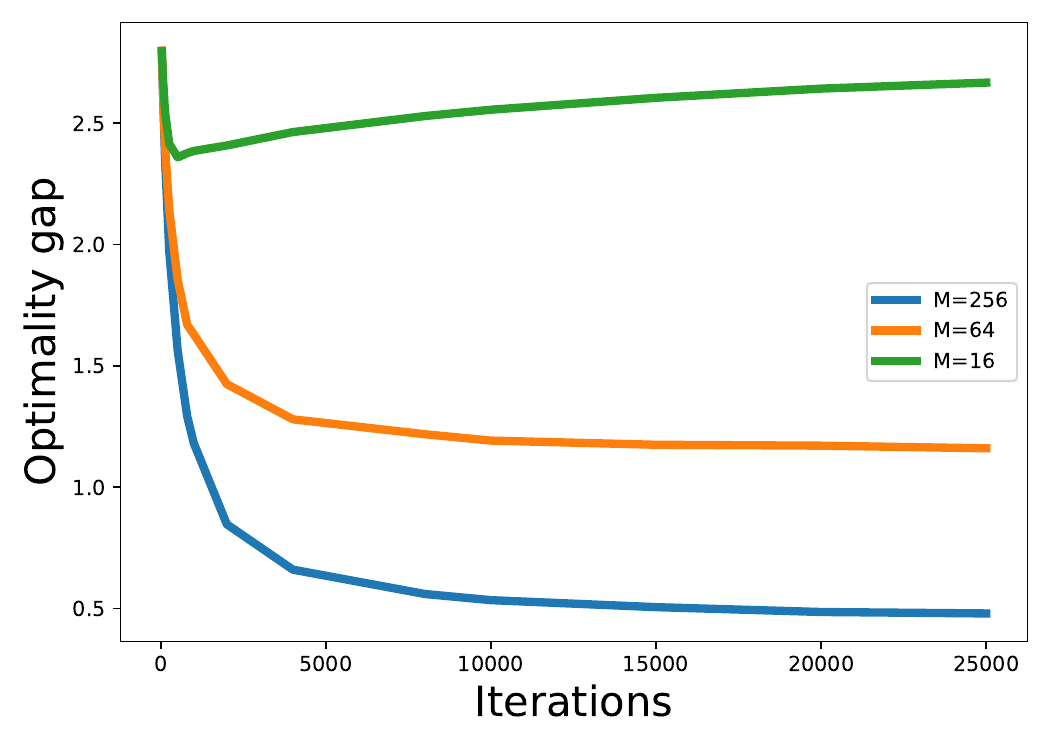}
        \caption{Optimality gap of ZPGPD-HF}
        \label{fig:sub1}
    \end{subfigure}
    \hfill
    \begin{subfigure}{0.45\textwidth}
        \centering
        \includegraphics[width=\linewidth]{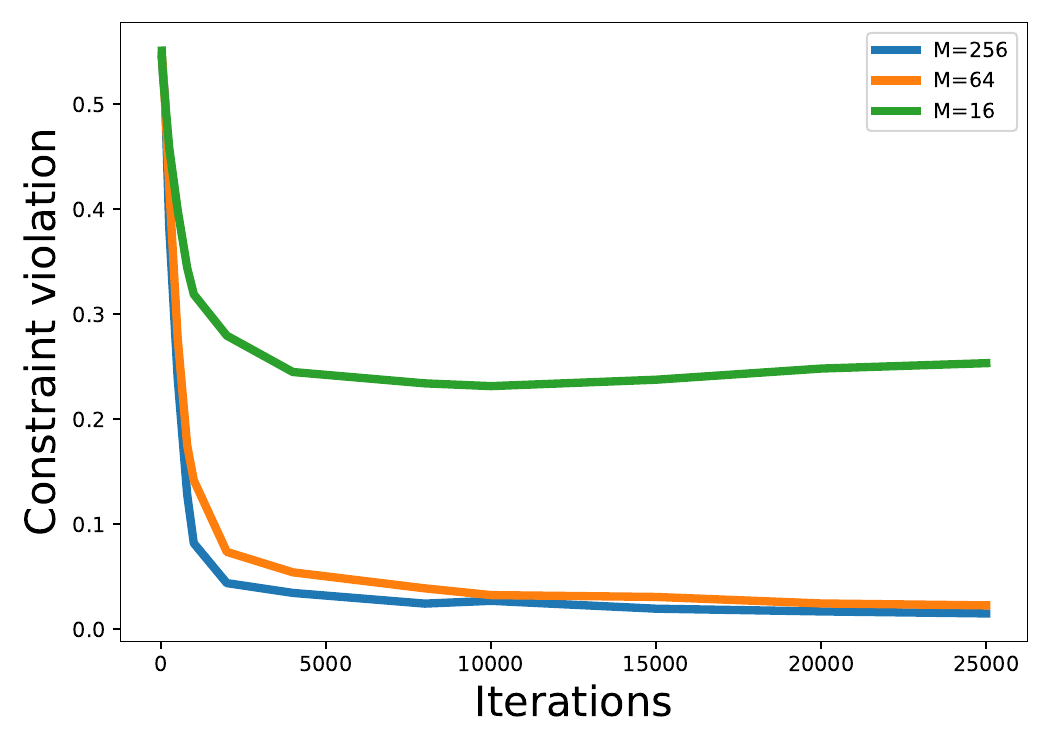}
        \caption{Constraint violation of ZPGPD-HF}
        \label{fig:sub2}
    \end{subfigure}
    \label{figs}  
\caption{Comparison of NPGPD-HF and ZPGPD-HF. In these experiments, we truncate the horizon to a fixed length of $H = 80$ and vary the number of evaluators $M =16,64,256$. }
\label{experi}
\end{figure}
\section{Conclusion}
In this paper, we investigate safe RLHF in discounted infinite horizon CMDP and propose two primal-dual policy gradient algorithms, namely NPGPD-HF and ZPGPD-HF. In contrast to existing RLHF methods that rely on explicit reward inference or are restricted to episodic finite horizon settings, the proposed algorithms directly leverage human preference and safety feedback and admit global non-asymptotic convergence guarantees. Our analysis characterizes the effects of optimization error, feedback inaccuracy, and horizon truncation, thereby providing a principled understanding of the tradeoffs involved in safe RLHF under infinite horizon discounted objectives. 

Future directions include extending the theory to large scale function approximation, deep neural policies, and richer safety settings such as multiple constraints, state dependent requirements, and partial observability, as well as validating the framework on realistic language model alignment tasks.
\newpage
%%%%%%%%%%%%%%%%%%%%%%%%%%%%%%%%%%%%%%%%%%%%%%%%%%%%%%%%%%%%%%%%
%% NOTE: THIS MARKS THE END OF THE "MAIN TEXT"
%%%%%%%%%%%%%%%%%%%%%%%%%%%%%%%%%%%%%%%%%%%%%%%%%%%%%%%%%%%%%%%%

%%%%%%%%%%%%%%%%%%%%%%%%%%%%%%%%%%%%%%%%%%%%%%%%%%%%%%%%%%%%%%%%
%% Bibliography
%%%%%%%%%%%%%%%%%%%%%%%%%%%%%%%%%%%%%%%%%%%%%%%%%%%%%%%%%%%%%%%%
\bibliography{main}
\bibliographystyle{rlj}

%%%%%%%%%%%%%%%%%%%%%%%%%%%%%%%%%%%%%%%%%%%%%%%%%%%%%%%%%%%%%%%%
% AUTHOR: If your paper has no supplementary materials, you may 
%         comment out the line below, which creates the title for
%         the supplementary materials.
%%%%%%%%%%%%%%%%%%%%%%%%%%%%%%%%%%%%%%%%%%%%%%%%%%%%%%%%%%%%%%%%
\beginSupplementaryMaterials
\appendix
\section{Related work}
\textbf{RLHF.}
Research on Reinforcement Learning from Human Feedback (RLHF) builds on several earlier paradigms in reinforcement learning and human–AI interaction. It extends preference-based reinforcement learning (PbRL) \citep{akrour2011preference}, which learns from pairwise preferences rather than numeric rewards, and integrates insights from human-in-the-loop reinforcement learning, reward shaping \citep{koppelman2006self}, and inverse reinforcement learning \cite{arora2021survey}. Compared with these precursors, RLHF generalizes feedback modalities to include critiques, corrections, and natural language instructions, while maintaining scalability through reward modeling and deep RL. Early PbRL focused mainly on direct policy optimization, whereas modern RLHF emphasizes reward model learning and active feedback collection to make human supervision more efficient.

RLHF’s prominence has grown through applications in large language model alignment, notably in ChatGPT and InstructGPT \citep{ouyang2022training}. Most theoretical RLHF papers \citep{zhan2023provable,zhu2023principled,wu2023making,du2024exploration} rely on reward inference methods that infer per-step rewards from human feedback (Christiano et al., 2017). These methods typically use the Bradley–Terry model \citep{bradley1952rank} to maximize the likelihood of trajectory comparisons. However, reward inference methods suffer from issues such as model misspecification, lack of ground-truth evaluation, distribution shift, and overfitting in joint reward–policy training \citep{casper2023open}. To overcome these drawbacks, \citep{rafailov2023direct} proposed Direct Preference Optimization (DPO), which directly fine-tunes models from human preferences, offering a simpler and more robust alternative but its theoretical justification only works for the bandit setting. \citep{zhang2024zeroth} employs the inverse of the preference function to estimate the difference between value functions, which serves as a fundamental basis for this paper.

\textbf{Constrained Markov Decision Process.}
CMDP is one of important mathematical foundation for safe reinforcement learning (safe RL) by extending standard MDP with cost functions and constraint bounds that formalize safety requirements. \citep{altman2021constrained} established the existence and structure of optimal randomized policies under constraints and revealed the computational hardness of feasibility checking for deterministic policies. Building upon this theoretical groundwork, modern studies reformulate safe RL as a constrained optimization problem solved via primal–dual and Lagrangian frameworks. Algorithms based on primal-dual methods have been widely adopted for solving constrained RL problems \citep{achiam2017constrained, tessler2018reward,chow2018risk}. \citep{paternain2019constrained} further proved a zero duality gap for CMDP, offering strong guarantees for dual-domain optimization. Subsequent advances such as UC-CFH \citep{kalagarla2021sample}, PG-PD \citep{ding2022ACC} and NPG-PD \citep{ding2020natural} established finite-sample and convergence analyses, bridging theoretical rigor and deep learning practice. 

\citep{dai2023safe} conducted an empirical study on Safe RLHF, bridging RLHF with safe reinforcement learning to develop a large language model that is helpful, safe, and responsive, 
which is another key inspiration for this paper.
\newpage
\section{Zeroth-Order Policy Gradient Primal-Dual Method from Human Feedback (ZPGPD-HF)}
\label{ZPGAP}
\begin{algorithm}[h]
\caption{Zeroth-Order Policy Gradient Primal-Dual Method from Human Feedback (ZPGPD-HF) }
\label{alg:ZPGPD-HF}
\textbf{Parameters:} initial parameter ${\boldsymbol{\theta}}_0$, learning rate $\alpha$, trim size $\Delta$, perturbation distance $\mu$.
\begin{algorithmic}[1]
\For{$t = 1:T$}
    \State sample $\boldsymbol{v}^{(t)}$ uniformly from a unit sphere $\mathbb{S}^{d-1} = \left\{ \boldsymbol{v} \in \mathbb{R}^d \,\middle|\, \|\boldsymbol{v}\|_2 = 1 \right\};$
    \State obtain a perturbed policy $\pi_{{\boldsymbol{\theta}}^{(t)} + \mu \boldsymbol{v}^{(t)}};$
    \For{$n = 1:N$}
        \State Draw $s_{n,0} \sim \rho$
        \State Sample a trajectory $\tau_{n,0}(s_{n,0}) \sim \pi_{{\boldsymbol{\theta}}^{(t)}}$ with initial state $s_{n,0}$ and length $H$;
        \State Sample a trajectory $\tau_{n,1}(s_{n,0}) \sim \pi_{{\boldsymbol{\theta}}^{(t)} + \mu \boldsymbol{v}^{(t)}}$ with initial state $s_{n,0}$ and length $H$;
        \State Query $M$ human evaluators about helpfulness with $(\tau_{n,1}(s_{n,0}), \tau_{n,0}(s_{n,0}))$ and obtain feedback $[o_{r,n,1}, \cdots, o_{r,n,M}];$
        
        \State Query $M$ human evaluators with $(\tau_{n,1}(s_{n,0}), \tau_{n,0}(s_{n,0}))$ about harmlessness and obtain feedback $[o_{g,n,1}, \cdots, o_{g,n,M}];$
        \State Estimate preference probability:
        \[
        p^{(t)}_{r,n}(\rho) = \text{clip}\left( \sum_{m=1}^M \frac{o_{r,n,m}}{M} ,\sigma_r(-G(H)),\sigma_r(G(H))\right);\]
        \[
        p^{(t)}_{g,n}(\rho) = \text{clip}\left( \sum_{m=1}^M \frac{o_{g,n,m}}{M}  ,\sigma_g(-G(H)),\sigma_g(G(H)) \right);
        \]
        \State Query $M$ human evaluators with $\tau_{n,0}$ to determine whether the response is harmless and obtain feedback $[o_{c,n,1}, \cdots, o_{c,n,M}];$
        
        \State Estimate absolute probability:
         \[
        p^{(t)}_{c,n}(\rho) = \text{clip}\left(  \sum_{m=1}^M \frac{o_{c,n,m}}{M} ,\sigma_c(-G(H)),\sigma_c(G(H)) \right);
        \]
    \EndFor
    \State Estimate the gradient reward function and utility function. 
    \[
    \hat{\boldsymbol{h}}^{(t)}_r(\rho) = \frac{d}{\mu} \frac{1}{N} \sum_{n=1}^N \sigma_r^{-1}(p^{(t)}_{r,n}) \boldsymbol{v}^{(t)};\quad
    \hat{\boldsymbol{h}}^{(t)}_g(\rho) = \frac{d}{\mu} \frac{1}{N} \sum_{n=1}^N \sigma_g^{-1}(p^{(t)}_{g,n}) \boldsymbol{v}^{(t)};
    \]
    \State Estimate the utility value:
    \[
    \hat{h}^{(t)}_c(\rho) =  \frac{1}{N} \sum_{n=1}^N \sigma_c^{-1}(p^{(t)}_{c,n});
    \]
    \State Update ${\boldsymbol{\theta}}^{(t+1)} = \mathcal{P}_{\boldsymbol{\Theta}} \left({\boldsymbol{\theta}}^{(t)} + \eta_1 (\hat{\boldsymbol{h}}^{(t)}_r+\lambda^{(t)}\hat{\boldsymbol{h}}^{(t)}_g)\right)$;
    \State Update $\lambda^{(t+1)} = \mathcal{P}_{\Lambda}\left(\lambda^{(t)} -\eta_2 \hat{h}_c^{(t)}(\rho) \right);$
\EndFor
\end{algorithmic}
\end{algorithm}
\newpage
\section{Supporting Lemmas}
In this section, we will first present the bias and variance introduced by the finite horizon approximation in Lemma~\ref{Lemma:finite horizon value}, \ref{Lemma:RVRQ}, \ref{Lemma:bias of fhr}, \ref{Lemma:var of fh}.

\begin{Lemma}
\label{Lemma:finite horizon value}
For any policy $\pi_{\boldsymbol{\theta}}$ and any initial state $s_0=s$, the bias between the finite horizon value function and the infinite horizon value function is bounded as
\begin{equation*}
\big\|  V^{\boldsymbol{\theta},H}_{\diamond}(s)  - V_{\diamond}^{\boldsymbol{\theta}}(s) \big\| 
\leq \underbrace{\frac{\gamma^H}{1 - \gamma}}_{\triangleq\,\epsilon_V(H)}.
\end{equation*}
Moreover, if the initial state $s_0$ is drawn from a distribution $\rho$, then
\begin{equation*}
\big\|  V^{\boldsymbol{\theta},H}_{\diamond}(\rho)  - V_{\diamond}^{\boldsymbol{\theta}}(\rho) \big\| 
\leq \epsilon_V(H).
\end{equation*}
\end{Lemma}

\begin{proof}
We first establish the bound for a fixed initial state $s$. By the definitions of the finite horizon and infinite horizon value functions, we have
\begin{equation*}
\begin{aligned}
\big\| V^{\boldsymbol{\theta},H}_{r}(s)  - V_{r}^{\boldsymbol{\theta}}(s) \big\|
&= \Big\| \mathbb{E} \Big[ \sum_{t=H}^{\infty} \gamma^{t} r(s^{(t)}, a^{(t)}) \,\big|\, \pi_{\boldsymbol{\theta}},\, s_0 = s \Big] \Big\|.
\end{aligned}
\end{equation*}
By the fact that rewards are uniformly bounded by $|r(s,a)| \le 1$, it follows that
\begin{equation*}
\big\| V^{\boldsymbol{\theta},H}_{r}(s)  - V_{r}^{\boldsymbol{\theta}}(s) \big\|
\le \sum_{t=H}^{\infty} \gamma^{t} = \frac{\gamma^H}{1 - \gamma}.
\end{equation*}

For the case where the initial state is drawn from a distribution $\rho$, taking the expectation over $s_0 \sim \rho$ yields
\begin{equation*}
\begin{aligned}
\big\| V^{\boldsymbol{\theta},H}_{r}(\rho)  - V_{r}^{\boldsymbol{\theta}}(\rho) \big\|
&= \Big\| \mathbb{E}_{s_0 \sim \rho} \Big[ \mathbb{E} \Big[ \sum_{t=H}^{\infty} \gamma^{t} r(s_{t}, a_{t}) \,\big|\, \pi_{\boldsymbol{\theta}},\, s_0 \Big] \Big] \Big\| \\
&\le \sum_{t=H}^{\infty} \gamma^{t} = \frac{\gamma^H}{1 - \gamma}.
\end{aligned}
\end{equation*}
Thus, in both cases, the truncation error introduced by a finite horizon $H$ is bounded by $\epsilon_V(H) = \gamma^H / (1 - \gamma)$.
\end{proof}
\begin{Lemma}
\label{Lemma:RVRQ}
Let $\tau(s,a)$ denote a trajectory of length $H$ generated under policy $\pi_{\boldsymbol{\theta}}$, initialized from $(s,a)\in\mathcal{S}\times\mathcal{A}$, and let $\tau(s)$ denote a trajectory of length $H$ generated under $\pi_{\boldsymbol{\theta}}$, initialized from $s\in\mathcal{S}$. Then, for any $\pi_{\boldsymbol{\theta}}$ and discount factor $\gamma \in (0,1)$, the truncation bias between the finite horizon expected return and the corresponding infinite horizon value functions satisfies
\begin{equation*}
\big\| \mathbb{E}\!\left[R^H_{\diamond}(\tau(s,a))
- Q^{\boldsymbol{\theta}}_{\diamond}(s,a)\right] \big\|
\le \epsilon_V(H),
\end{equation*}
and
\begin{equation*}
\big\| \mathbb{E}\!\left[R^H_{\diamond}(\tau(s))
- V^{\boldsymbol{\theta}}_{\diamond}(s)\right] \big\|
\le \epsilon_V(H).
\end{equation*}
\end{Lemma}
\begin{proof}
Consider first the case where the initial state–action pair $(s,a)$ is fixed. Then,
\begin{equation*}
\begin{aligned}
\big\| \mathbb{E}\!\left[ R^H_{\diamond}(\tau(s,a)) \right]
- Q^{\boldsymbol{\theta}}_{\diamond}(s,a) \big\|
&= \Big\| \mathbb{E}\!\left[ \sum_{t=H}^{\infty} \gamma^{t}
r_{\diamond}\!\left(s_{t}, a_{t}\right)
\,\middle|\, s_{0}=s,\, a_{0}=a \right] \Big\| \\
&\le \sum_{t=H}^{\infty} \gamma^{t}
= \frac{\gamma^{H}}{1 - \gamma}
=\epsilon_V(H).
\end{aligned}
\end{equation*}

Next, if only the initial state $s$ is fixed, we have
\begin{equation*}
\begin{aligned}
\big\| \mathbb{E}\!\left[ R^H_{\diamond}(\tau(s)) \right]
- V^{\boldsymbol{\theta}}_{\diamond}(s) \big\|
&= \Big\| \mathbb{E}\!\left[ \sum_{t=H}^{\infty} \gamma^{t}
r_{\diamond}\!\left(s_{t}, a_{t}\right)
\,\middle|\, s_{0}=s \right] \Big\| \\
&\le \sum_{t=H}^{\infty} \gamma^{t}
= \frac{\gamma^{H}}{1 - \gamma}
= \epsilon_V(H).
\end{aligned}
\end{equation*}
\end{proof}
\begin{Lemma}
\label{Lemma:bias of fhr}
Let $\tau_{n,0}$, $n = 1, 2, \ldots, N$, denote independent trajectories generated under policy $\pi_{\boldsymbol{\theta}}$ with initial states $s_0 \sim \rho$. Then, the bias between the expected average finite horizon return and the infinite horizon value function is bounded as
\begin{equation*}
\bigg\|
\mathbb{E}\!\left[
\frac{1}{N}\sum_{n=1}^N R^H_g(\tau_{n,0})
- V^{\boldsymbol{\theta}}_g(\rho)
\right]
\bigg\|
\le \epsilon_V(H).
\end{equation*}
\end{Lemma}

\begin{proof}
We begin by adding and subtracting $V^{\boldsymbol{\theta},H}_g(\rho)$ within the expectation and applying the triangle inequality.
\begin{equation*}
\begin{aligned}
\bigg\|
\mathbb{E}\!\left[
\frac{1}{N}\sum_{n=1}^N R^H_g(\tau_{n,0})
- V^{\boldsymbol{\theta}}_g(\rho)
\right]
\bigg\|
&= 
\Big\|
\mathbb{E}\!\left[
\frac{1}{N}\sum_{n=1}^N
\big(R^H_g(\tau_{n,0}) - V^{\boldsymbol{\theta},H}_g(\rho)\big)
+ V^{\boldsymbol{\theta},H}_g(\rho) - V^{\boldsymbol{\theta}}_g(\rho)
\right]
\Big\| \\[4pt]
&\overset{(a)}{\le}
\Big\|
\mathbb{E}\!\left[
R^H_g(\tau_{n,0}) - V^{\boldsymbol{\theta},H}_g(\rho)
\right]
\Big\|
+
\big\|
V^{\boldsymbol{\theta},H}_g(\rho) - V^{\boldsymbol{\theta}}_g(\rho)
\big\|\\
&\overset{(b)}{\leq}\epsilon_V(H).
\end{aligned}
\end{equation*}

In step~$(a)$, the inequality follows from the Cauchy–Schwarz inequality and the independence of trajectories, which ensures that the cross terms vanish.  
For the first term of $(b)$, by the definition of the finite horizon value function $V^{\boldsymbol{\theta},H}_g(\rho)$, we have
\[
\mathbb{E}\!\left[ R^H_g(\tau_{n,0}) \right]
= V^{\boldsymbol{\theta},H}_g(\rho),
\]
so the term equals zero.  
For the second term of $(b)$, Lemma~\ref{Lemma:finite horizon value} implies that
\[
\big\|
V^{\boldsymbol{\theta},H}_g(\rho) - V^{\boldsymbol{\theta}}_g(\rho)
\big\|
\le \epsilon_V(H),
\]
which completes the proof.
\end{proof}

\begin{Lemma}
\label{Lemma:var of fh}
Let $\tau_{n,0}$, $n=1,2,\ldots,N$, be independent trajectories sampled under policy $\pi_{\boldsymbol{\theta}}$ with initial state $s_0 \sim \rho$. Then, the variance between the average trajectory reward and the infinite horizon value function satisfies
\begin{equation*}
\mathbb{E}\!\left[\left\|
\frac{1}{N}\sum_{n=1}^N R^H_g(\tau_{n,0})
- V_g^{\boldsymbol{\theta}}(\rho)
\right\|_2^2\right]
\le
\underbrace{\frac{2G(H)^2}{N} + 2\epsilon_V^2(H)}_{\triangleq \sigma_1(H,N)}.
\end{equation*}
\end{Lemma}

\begin{proof}
The variance between a single trajectory return and the finite horizon value function is
\begin{equation*}
\label{eq:rv}
\begin{aligned}
&\mathbb{E}\!\left[\Big\|\sum_{n=1}^N
\big(R^H_g(\tau_{n,0}) - V_g^{\boldsymbol{\theta},H}(\rho)\big)\Big\|_2^2\right] \\
&= \sum_{n=1}^N
\mathbb{E}\!\left[\|R^H_g(\tau_{n,0}) - V_g^{\boldsymbol{\theta},H}(\rho)\|_2^2\right]
+ \sum_{i\neq j}
\mathbb{E}\!\left[\|R^H_g(\tau_{i,0}) - V_g^{\boldsymbol{\theta},H}(\rho)\|_2
\|R^H_g(\tau_{j,0}) - V_g^{\boldsymbol{\theta},H}(\rho)\|_2\right] \\
&\overset{(a)}{=}
\sum_{n=1}^N
\mathbb{E}\!\left[\|R^H_g(\tau_{n,0}) - V_g^{\boldsymbol{\theta},H}(\rho)\|_2^2\right]\\
&\quad+ \sum_{i\neq j}
\mathbb{E}\!\left[\|R^H_g(\tau_{i,0}) - V_g^{\boldsymbol{\theta},H}(\rho)\|_2\right]
\mathbb{E}\!\left[\|R^H_g(\tau_{j,0}) - V_g^{\boldsymbol{\theta},H}(\rho)\|_2\right] \\
&\overset{(b)}{=}
N\,\mathbb{E}\!\left[\|R^H_g(\tau_{n,0}) - V_g^{\boldsymbol{\theta},H}(\rho)\|_2^2\right],
\end{aligned}
\end{equation*}
where in (a) independence of trajectories is applied, and (b) follows from the definition of $V_g^{\boldsymbol{\theta},H}(\rho)$.

Then,
\begin{equation*}
\begin{aligned}
&\mathbb{E}\!\left[\Big\|
\frac{1}{N}\sum_{n=1}^N R^H_g(\tau_{n,0})
- V_g^{\boldsymbol{\theta}}(\rho)
\Big\|_2^2\right] \\
&=
\mathbb{E}\!\left[\Big\|
\frac{1}{N}\sum_{n=1}^N
\big(R^H_g(\tau_{n,0}) - V_g^{\boldsymbol{\theta},H}(\rho)\big)
+ V_g^{\boldsymbol{\theta},H}(\rho)
- V_g^{\boldsymbol{\theta}}(\rho)
\Big\|_2^2\right] \\
&\overset{(a)}{\le}
\frac{2}{N^2}
\mathbb{E}\!\left[
\Big\|\sum_{n=1}^N
\big(R^H_g(\tau_{n,0}) - V_g^{\boldsymbol{\theta},H}(\rho)\big)
\Big\|_2^2\right]
+ 2\|V_g^{\boldsymbol{\theta},H}(\rho)
- V_g^{\boldsymbol{\theta}}(\rho)\|_2^2 \\
&\overset{(b)}{\le}
\frac{2}{N}\,
\mathbb{E}\!\left\|
\sum_{t=0}^{H}\gamma^t g(s_{t},a_{t})
\right\|_2^2
+ 2\epsilon_V^2(H) \\
&\overset{(c)}{\le}
\frac{2G^2}{N} + 2\epsilon_V^2(H),
\end{aligned}
\end{equation*}
where $(a)$ follows from the Cauchy–Schwarz inequality and the independence of trajectories,  
$(b)$ follows from Lemma~\ref{Lemma:finite horizon value} and the inequality $\operatorname{Var}(X)\le\mathbb{E}[X^2]$,  
and (c) is true because of the boundedness of $g(s_{t},a_{t})$.
\end{proof}

\begin{comment}
\begin{Lemma}[Variance of policy gradient] the variance of the corresponding policy gradients are bounded by
\begin{equation}
 \mathbb{E} \left[ \left\|\mathbb{E}\|\nabla_{\boldsymbol{\theta}} V^{{\boldsymbol{\theta}}^{(t)},H}_{\diamond}(\rho)\| - \nabla_{\boldsymbol{\theta}} V^{{\boldsymbol{\theta}}^{(t)}}_\diamond(\rho)\right\|^2  \right]
\leq 2\epsilon_1^2(H)
\end{equation}

\begin{equation}
 \mathbb{E} \left[\left\|  \mathbb{E}\|V^{{\boldsymbol{\theta}}^{(t)},H}_{\diamond}(\rho)\|  - V_{\diamond}^{{\boldsymbol{\theta}}^{(t)}} (\rho) \right\|^2\right] 
\leq 2\epsilon_2^2\bigl(H\bigr)
\end{equation}

\end{Lemma}
\end{comment}

We next provide the bias and variance bounds for estimating reward or utility differences induced by human feedback, which is established in \citet{zhang2024zeroth}.

\begin{Lemma}
\label{Lemma:human_feedback_bias_var}
Let Assumption~\ref{assum:link} holds. 
For any pair of trajectories $(\tau_{n,0}, \tau_{n,1})$ evaluated by $M$ independent human evaluators, the term $\sigma_\diamond^{-1}(p^{(t)}_{\diamond,n})$ generated in Algorithm~\ref{alg:NPGPD-HF} or Algorithm~\ref{alg:ZPGPD-HF} satisfies the following bounds:
\begin{equation*}
\label{eq:human_bias}
\mathbb{E}\!\left[
\big|\, \sigma_\diamond^{-1}(p^{(t)}_{\diamond,n})
- \big(R_\diamond^H(\tau_{n,1}) - R_\diamond^H(\tau_{n,0})\big)
\big|
\right]
\le
L \sqrt{\frac{2 \log M}{M}}
+ \frac{2G(H)}{M^2},
\end{equation*}
and
\begin{equation*}
\label{eq:human_var}
\mathbb{E}\!\left[
\big|\, \sigma_\diamond^{-1}(p^{(t)}_{\diamond,n})
- \big(R_\diamond^H(\tau_{n,1}) - R_\diamond^H(\tau_{n,0})\big)
\big|^2
\right]
\le
\frac{2L^2 \log M}{M}
+ \frac{4G(H)^2}{M^2}.
\end{equation*}
\end{Lemma}

The absolute value of the utility function can similarly be bounded by the same constant.

\begin{Lemma}
\label{Lemma:sigmac}
Let Assumption~\ref{assum:link}holds.
For any trajectory $\tau_{n,0}$ evaluated by $M$ independent human evaluators, the term $\sigma_c^{-1}(p^{(t)}_{c,n})$ generated in Algorithm~\ref{alg:NPGPD-HF} or Algorithm~\ref{alg:ZPGPD-HF} satisfies the following bounds:
\begin{equation*}
\label{eq:sigmac_bias}
\mathbb{E}\!\left[
\big|\sigma_c^{-1}(p^{(t)}_{c,n}) - (R_c^H(\tau_{n,0}) - b)\big|
\right]
\le
L \sqrt{\frac{2 \log M}{M}}
+ \frac{2G(H)}{M^2},
\end{equation*}
and
\begin{equation*}
\label{eq:sigmac_var}
\mathbb{E}\!\left[
\big|\sigma_c^{-1}(p^{(t)}_{c,n}) - (R_c^H(\tau_{n,0}) - b)\big|^2
\right]
\le
\frac{2L^2 \log M}{M}
+ \frac{4G(H)^2}{M^2}.
\end{equation*}
\end{Lemma}
\begin{comment}
\begin{Lemma}
For any $i$, the gradient norm is bounded as
\begin{equation}
\begin{aligned}
&\left\| \nabla_{\bm{\boldsymbol{\theta}}_i} f_i(\bm{\boldsymbol{\theta}}_i, \mathcal{U}^r \bm{\lambda}_i) \right\| \\
&\leq \mathbb{E} \left[ 
\left\| \sum_{t=0}^{\infty} \left( \sum_{\tau=0}^t \nabla \log \pi(a_i^\tau | \bm{s}^\tau; \bm{\boldsymbol{\theta}}_i) \right) 
\gamma^t R_i(\bm{s}^t, \bm{a}^t)
- \sum_{t=0}^{\infty} \left\langle \left( \sum_{\tau=0}^t \nabla \log \pi(a_i^\tau | \bm{s}^\tau; \bm{\boldsymbol{\theta}}_i) \right) 
\gamma^t C_i(\bm{s}^t, \bm{a}^t), \mathcal{U}^r \bm{\lambda}_i \right\rangle \right\| 
\right] \\
&\leq \frac{G' G \left(1 + m \mathcal{U}^0 \bar{\sigma}_{\lambda}\right)}{(1 - \gamma)^2} 
\triangleq \widetilde{G},
\end{aligned}
\end{equation}
\end{Lemma}
\end{comment}

The following Lemma from \citep{agarwal2021theory} establishes the Lipschitz continuity of the gradient of the value function in the tabular setting.
\begin{Lemma}
\label{Lemma:smoothness}
Let $L'=\frac{2 \gamma |\mathcal{A}|}{(1 - \gamma)^{3}}$ , the gradient of infinite horizon value function is $L'$-Lipschitz continuous:
\begin{equation*}
\left\| \nabla_{\boldsymbol{\theta}} V^{\boldsymbol{\theta}}(\rho) - \nabla_{\boldsymbol{\theta}} V^{{\boldsymbol{\theta}}'}(\rho) \right\|_{2}
\leq L^{\prime} \left\| {\boldsymbol{\theta}} - {\boldsymbol{\theta}}' \right\|_{2}
\end{equation*}
\end{Lemma}
Then, for the smoothed function, \citet{liu2018zeroth} establishes the following properties.
\begin{Lemma}
\label{Lemma:zeroth}
Let $d=|S||A|$,  $L'=\frac{2 \gamma |\mathcal{A}|}{(1 - \gamma)^{3}}$
and  \(V_{\diamond}^{\mu,{\boldsymbol{\theta}}}(\rho) = \mathbb{E}_{\boldsymbol{v}'} \left[V_\diamond^{{\boldsymbol{\theta}}+\mu\boldsymbol{v}'}(\rho) \right]\)

With Lemma \ref{Lemma:smoothness}, suppose $\boldsymbol{v}$ is sampled from a uniform distribution over a unit sphere and $\boldsymbol{v}'$ is sampled from a uniform distribution over a unit ball both in $d$-dimensional space. Then the smoothed value function $V_{\diamond}^{\mu,{\boldsymbol{\theta}}^{(t)}}(\rho)$ defined above satisfies:

\begin{enumerate}
    \item $V_{\diamond}^{\mu,{\boldsymbol{\theta}}^{(t)}}(\rho)$ is $L'$-smooth and satisfies:
    \begin{equation*}
    \nabla_{\boldsymbol{\theta}} V_{\diamond}^{\mu,{\boldsymbol{\theta}}}(\rho) = \mathbb{E}_{\boldsymbol{v}} \left[ \frac{d}{\mu} \left( V_\diamond^{{\boldsymbol{\theta}}+\mu\boldsymbol{v}}(\rho) - V_\diamond^{\boldsymbol{\theta}}(\rho) \right) \boldsymbol{v} \right].
    \end{equation*}
    
    \item For any ${\boldsymbol{\theta}} \in \mathbb{R}^d$, the function value difference satisfies:
    \begin{equation*}
    \left| V_{\diamond}^{\mu,{\boldsymbol{\theta}}}(\rho) - V_{\diamond}^{\boldsymbol{\theta}}(\rho) \right| \le \frac{L' \mu^2}{2}.
    \end{equation*}
    
    \item For any ${\boldsymbol{\theta}} \in \mathbb{R}^d$, the gradient difference satisfies:
    \begin{equation*}
    \left\| \nabla_{\boldsymbol{\theta}} V_{\diamond}^{\mu,{\boldsymbol{\theta}}}(\rho) - \nabla_{\boldsymbol{\theta}} V_{\diamond}^{\boldsymbol{\theta}}(\rho) \right\|_2 \le \frac{\mu L' d}{2}.
    \end{equation*}
    
    \item For any ${\boldsymbol{\theta}} \in \mathbb{R}^d$, the gradient noise satisfies:
    \begin{equation*}
    \mathbb{E}_{\boldsymbol{v}} \left[ \left\| \frac{d}{\mu} \left( V_{\diamond}^{\mu,{\boldsymbol{\theta}}}(\rho) - V_{\diamond}^{\boldsymbol{\theta}}(\rho) \right) \boldsymbol{v} \right\|_2^2 \right] 
    \le 2d \left\| \nabla_{\boldsymbol{\theta}} V_{\diamond}^{\boldsymbol{\theta}}(\rho) \right\|_2^2 + \frac{\mu^2 L'^2 d^2}{2}.
    \end{equation*}
\end{enumerate}
\end{Lemma}

With the bias of of finite horizon approximation and human feedback, we are ready to derive bounds on the bias of the value function of utility and its estimation.
The following Lemma form \cite{ding2020natural} is essential for proving bound of constraint violence. 
\begin{Lemma}[Constraint violation]
\label{Lemma:constraint_violation}
Let Assumption~\ref{assum:slater} hold. For any $C \ge 2\lambda^\star$, if there exists a policy $\pi \in \Pi$ and $\delta > 0$ such that 
\[
V_r^\star(\rho) - V_r^\pi(\rho) + C[b - V_g^\pi(\rho)]_+ \le \delta,
\]
then 
\[
[b - V_g^\pi(\rho)]_+ \le \frac{2\delta}{C},
\]
where $[x]_+ = \max(x, 0)$.
\end{Lemma}

\begin{Lemma}
\label{Lemma:bias hc}
The bias between $h_c^{(t)}(\rho)$, which is generated at line~20 of Algorithm~\ref{alg:NPGPD-HF} or line~15 of Algorithm~\ref{alg:ZPGPD-HF}, and the infinite horizon value function satisfies
\begin{equation*}
\big\|
\mathbb{E}\!\big[V_g^{\boldsymbol{\theta}^{(t)}}(\rho) - b - \hat h_c^{(t)}(\rho)\big]
\big\|
\le
\underbrace{
\epsilon_V(H)
+ \left(
L \sqrt{\tfrac{2 \log M}{M}}
+ \tfrac{2G(H)}{M^2}
\right)
}_{\triangleq\, \epsilon_c(H,M)}.
\end{equation*}
\end{Lemma}

\begin{proof}
We have
\begin{align*}
&\big\|
\mathbb{E}\!\big[V_g^{\boldsymbol{\theta}^{(t)}}(\rho) - b - \hat h_c^{(t)}(\rho)\big]
\big\|\\
&=
\Big\|
\mathbb{E}\!\left[
V_g^{\boldsymbol{\theta}^{(t)}}(\rho)
- \tfrac{1}{N}\!\sum_{n=1}^N R_g^H(\tau_{n,0})
+ \tfrac{1}{N}\!\sum_{n=1}^N (R_g^H(\tau_{n,0}) - b)
- \hat h_c^{(t)}(\rho)
\right]
\Big\| \\
&\overset{(a)}{\le}
\Big\|
\mathbb{E}\!\left[
V_g^{\boldsymbol{\theta}^{(t)}}(\rho)
- \tfrac{1}{N}\!\sum_{n=1}^N R_g^H(\tau_{n,0})
\right]
\Big\|
+
\Big\|
\mathbb{E}\!\left[
\tfrac{1}{N}\!\sum_{n=1}^N (R_g^H(\tau_{n,0}) - b)
- \tfrac{1}{N}\!\sum_{n=1}^N \sigma_c^{-1}(p^{(t)}_{c,n})
\right]
\Big\| \\
&\overset{(b)}{\le}
\epsilon_V(H)
+
\left(
L \sqrt{\tfrac{2 \log M}{M}}
+ \tfrac{2G(H)}{M^2}
\right),
\end{align*}
where $(a)$ follows from the triangle inequality,  
and $(b)$ is true due to Lemma~\ref{Lemma:bias of fhr}, the Cauchy–Schwarz inequality, and Lemma~\ref{Lemma:sigmac}.
\end{proof}
Bounds on the bias of the gradient of the value function with respect to reward or utility can also be derived.
\begin{Lemma}
\label{Lemma:overall}
Conditioned on the information filtration $\mathcal{F}_t$ at any time $t$, the bias between $h_r^{(t)}(\rho)$ or $h_g^{(t)}(\rho)$, generated at line~14 of Algorithm~\ref{alg:ZPGPD-HF}, and the infinite horizon policy gradient value function satisfies
\begin{align*}
\big\|
\mathbb{E}\!\big[
\nabla_{\boldsymbol{\theta}} V_\diamond^{\boldsymbol{\theta}^{(t)}}(\rho)
- \hat{\boldsymbol{h}}_\diamond^{(t)}(\rho)
\,\big|\,
\mathcal{F}_t
\big]
\big\|
\le
\underbrace{
\frac{\mu L'd}{2}
+ \frac{2d\,\epsilon_V(H)}{\mu}
+ \frac{d}{\mu}
\left(
L \sqrt{\tfrac{2 \log M}{M}}
+ \tfrac{2G(H)}{M^2}
\right)
}_{\triangleq\,\epsilon_{rg}(H,M)}.
\end{align*}
\end{Lemma}
\begin{proof}
The proof begins by establishing the following two equalities.

The first equality is given by
\begin{equation}
\label{eqa:inverse-sigma}
\begin{aligned}
\mathbb{E}\biggl[\frac{d}{\mu}\sum_{n=1}^N\frac{\sigma_\diamond^{-1}(p^t_{\diamond,n})}{N}\boldsymbol{v}^{(t)} \;\bigg|\;\mathcal{F}_t\biggr]
&\overset{(a)}{=}\frac{d}{\mu}\,\mathbb{E}\biggl[\Bigl.\mathbb{E}\Bigl[\sum_{n=1}^N\frac{\sigma^{-1}(p_{t,n})}{N}\Bigr|\boldsymbol{v}^{(t)}\Bigr]\,\boldsymbol{v}^{(t)} \;\Bigm|\;\mathcal{F}_t\biggr]\\
&\overset{(b)}{=}\frac{d}{\mu}\,\mathbb{E}\bigl[\mathbb{E}[\sigma^{-1}(p_{t,n})\mid \boldsymbol{v}^{(t)}]\,\boldsymbol{v}^{(t)}\mid \mathcal{F}_t\bigr]\\
&\overset{}{=}\frac{d}{\mu}\,\mathbb{E}\bigl[\mathbb{E}\bigl[\sigma^{-1}(p_{t,n})-\bigl(R_\diamond^H(\tau_{n,1})-R_\diamond^H(\tau_{n,0})\bigr)\mid \boldsymbol{v}^{(t)}\bigr]\,\boldsymbol{v}^{(t)}\mid \mathcal{F}_t\bigr]\\
&\quad +\frac{d}{\mu}\,\mathbb{E}\bigl[\mathbb{E}[R_\diamond^H(\tau_{n,1})-R_\diamond^H(\tau_{n,0})\mid \boldsymbol{v}^{(t)}]\,\boldsymbol{v}^{(t)}\mid \mathcal{F}_t\bigr],
\end{aligned}
\end{equation}
where $(a)$ follows from the law of total expectation and $(b)$ is true because of the independence of the trajectories generated under the same policy pair.

The second equality is 
\begin{equation}
\label{eq:zeroH}
\begin{aligned}
&\mathbb{E}\bigl[V_\diamond^{{\boldsymbol{\theta}}^{(t)}+\mu\boldsymbol{v}^{(t)},H}(\rho) - V_\diamond^{{\boldsymbol{\theta}}^{(t)},H}(\rho) - (R^H_\diamond(\tau_{n,1}) - R^H_\diamond(\tau_{n,0}))
      \mid \boldsymbol{v}^{(t)}\bigr]\\
&= V_\diamond^{{\boldsymbol{\theta}}^{(t)}+\mu\boldsymbol{v}^{(t)},H}(\rho)
  - \mathbb{E}\bigl[R_\diamond^H(\tau_{n,1}) \mid \tau_{n,1}\sim \pi_{{\boldsymbol{\theta}^{(t)}}+\mu \boldsymbol{v}^{(t)}}\bigr]\\  
&\quad+ \mathbb{E}\bigl[R^H_\diamond(\tau_{n,0}) \mid \tau_{n,0}\sim \pi_{{\boldsymbol{\theta}}^{(t)}}\bigr]
  - V_\diamond^{{\boldsymbol{\theta}}^{(t)},H}(\rho)= 0.
\end{aligned}
\end{equation}

We then have
\begin{align*}
&\|\mathbb{E}\bigl[\nabla_{\boldsymbol{\theta}}V_\diamond^{{\boldsymbol{\theta}}^{(t)}}(\rho) - \hat{\boldsymbol{h}}_\diamond^{(t)}(\rho) \mid \mathcal{F}_t\bigr]\|\\
&=\|\mathbb{E}\bigl[\nabla_{\boldsymbol{\theta}}V_\diamond^{{\boldsymbol{\theta}}^{(t)}}(\rho)- \nabla_{\boldsymbol{\theta}}V_\diamond^{\mu,{\boldsymbol{\theta}}^{(t)}}(\rho)+\nabla_{\boldsymbol{\theta}}V_\diamond^{\mu,{\boldsymbol{\theta}}^{(t)}}(\rho)- \hat{\boldsymbol{h}}_\diamond^{(t)}(\rho) \mid \mathcal{F}_t\bigr]\|\\
&
\overset{(a)}{\leq} \|\nabla_{\boldsymbol{\theta}}V_\diamond^{{\boldsymbol{\theta}}^{(t)}}(\rho)- \nabla_{\boldsymbol{\theta}}V_\diamond^{\mu,{\boldsymbol{\theta}}^{(t)}}(\rho)\|\\
&\quad+ \left\|\mathbb{E} \left[ \frac{d \left( V_\diamond^{{\boldsymbol{\theta}}^{(t)}+\mu\boldsymbol{v}^{(t)}}(\rho) - V_\diamond^{{\boldsymbol{\theta}}^{(t)}}(\rho) \right)}{\mu} \boldsymbol{v}^{(t)} \,\middle|\, \mathcal{F}_t \right]-\mathbb{E}\biggl[\frac{d}{\mu}\sum_{n=1}^N\frac{\sigma_\diamond^{-1}(p^t_{\diamond,n})}{N}\boldsymbol{v}^{(t)} \;\bigg|\;\mathcal{F}_t\biggr]\right\|\\
&
\overset{(b)}{\leq}\frac{\mu L'd}{2} + \left\| \frac{d}{\mu}\mathbb{E} \left[ \left( V_\diamond^{{\boldsymbol{\theta}}^{(t)}+\mu\boldsymbol{v}^{(t)}}(\rho) - V_\diamond^{{\boldsymbol{\theta}}^{(t)}}(\rho)-(V_\diamond^{{\boldsymbol{\theta}}^{(t)}+\mu\boldsymbol{v}^{(t)},H}(\rho) - V_\diamond^{{\boldsymbol{\theta}}^{(t)},H}(\rho)) \right) \boldsymbol{v}^{(t)} \,\middle|\, \mathcal{F}_t \right]\right.\\
&\left.\quad+\frac{d}{\mu}\mathbb{E} \left[ \left(V_\diamond^{{\boldsymbol{\theta}}^{(t)}+\mu\boldsymbol{v}^{(t)},H}(\rho) - V_\diamond^{{\boldsymbol{\theta}}^{(t)},H}(\rho) \right) \boldsymbol{v}^{(t)} \,\middle|\, \mathcal{F}_t \right] -\mathbb{E}\biggl[\frac{d}{\mu}\sum_{n=1}^N\frac{\sigma_\diamond^{-1}(p^t_{\diamond,n})}{N}\boldsymbol{v}^{(t)} \;\bigg|\;\mathcal{F}_t\biggr]\right\|\\
&\overset{(c)}{\leq} 
\frac{\mu L'd}{2}\\
&\quad+\left\|\frac{d}{\mu}\mathbb{E}\left[\mathbb{E} \left[ \left( V_\diamond^{{\boldsymbol{\theta}}^{(t)}+\mu\boldsymbol{v}^{(t)}}(\rho) - V_\diamond^{{\boldsymbol{\theta}}^{(t)}}(\rho)-(V_\diamond^{{\boldsymbol{\theta}}^{(t)}+\mu\boldsymbol{v}^{(t)},H}(\rho) - V_\diamond^{{\boldsymbol{\theta}}^{(t)},H}(\rho)) \right) \mid\boldsymbol{v}^{(t)}\right] \boldsymbol{v}^{(t)}\,\middle|\, \mathcal{F}_t \right]\right.\\
&\left.\quad+\frac{d}{\mu}\,\mathbb{E}\Bigl[\mathbb{E}\bigl[V_\diamond^{{\boldsymbol{\theta}}^{(t)}+\mu\boldsymbol{v}^{(t)},H}(\rho) - V_\diamond^{{\boldsymbol{\theta}}^{(t)},H}(\rho) - (R^H_\diamond(\tau_{n,1}) - R^H_\diamond(\tau_{n,0}))
      \mid \boldsymbol{v}^{(t)}\bigr]\,
      \boldsymbol{v}^{(t)}
      \;\Bigm|\;\mathcal{F}_t
    \Bigr]\right.\\
&\left.\quad
  - \frac{d}{\mu}\,
    \mathbb{E}\Bigl[
      \mathbb{E}\bigl[\sigma^{-1}(p_{t,n}) 
         - \bigl(R^H_\diamond(\tau_{n,1}) - R^H_\diamond(\tau_{n,0})\bigr)
      \mid \boldsymbol{v}^{(t)}\bigr]\,
      \boldsymbol{v}^{(t)}
      \;\Bigm|\;\mathcal{F}_t
    \Bigr]\right\|.\\
&\overset{(d)}{\leq} \frac{\mu L'd}{2}+ \frac{2d\epsilon_V(H)}{\mu}
  + \frac{d}{\mu}\,
    \mathbb{E}\Bigl[
      \mathbb{E}\bigl[|\sigma^{-1}(p_{t,n}) 
         - \bigl(R^H_\diamond(\tau_{n,1}) - R^H_\diamond(\tau_{n,0})\bigr)|
      \mid \boldsymbol{v}^{(t)}\bigr]\,
      \boldsymbol{v}^{(t)}
      \;\Bigm|\;\mathcal{F}_t
    \Bigr]\\
&\overset{(e)}{\leq}\frac{\mu L'd}{2}+ \frac{2d\epsilon_V(H)}{\mu}+\frac{d}{\mu}\left( L \sqrt{ \frac{2 \log M}{M} } + \frac{2G(H)}{M^2}\right).
\end{align*}
In (a), the triangle inequality, Lemma~\ref{Lemma:zeroth}, and the definition of $\hat{\boldsymbol{h}}_\diamond^{(t)}$ are applied.  
In $(b)$, we applies Lemma \ref{Lemma:zeroth}.  
Step $(c)$ follows from the law of total expectation and Equation~(\ref{eqa:inverse-sigma}).  
And $(d)$ is true due to the triangle inequality, Jensen's inequality, Lemma \ref{Lemma:finite horizon value}, and Equation~(\ref{eq:zeroH}).  
Finally, we use Lemma \ref{Lemma:human_feedback_bias_var} in step $(e)$ to complete the proof.
\end{proof}
We now provide an bound on $(h_c^{(t)}(\rho))$.
\begin{Lemma}
The expectation of $(h_c^{(t)}(\rho))^2$, which is generated in line 20 of Algorithm~\ref{alg:NPGPD-HF} or line 15 of Algorithm~\ref{alg:ZPGPD-HF}, is upper bounded by
\label{Lemma:var hc}
\begin{equation*}
\mathbb{E}[(h^{(t)}_c(\rho))^2]\leq \underbrace{\frac{6L^2 \log M}{M} + \frac{12G(H)^2}{M^2}+\frac{6G(H)^2}{N}+6\epsilon_V^2(H)+ \frac{3}{(1-\gamma)^2}}_{\triangleq\sigma^2_2(H,M,N)}.
\end{equation*}
\end{Lemma}
\begin{proof}
\begin{align*}
\mathbb{E}\!\left[(h^{(t)}_c(\rho))^2\right]
&= \mathbb{E}\!\left[\left|\frac{1}{N}\sum_{n=1}^N \sigma_c^{-1}(p^{(t)}_{c,n})\right|^2\right] \\[1ex]
&\overset{(a)}{\le}
3\,\mathbb{E}\!\left[\left|\frac{1}{N}\sum_{n=1}^N \bigl(\sigma_c^{-1}(p^{(t)}_{c,n}) - (R^H_g(\tau_{n,0}) - b)\bigr)\right|^2\right] \\[0.5ex]
&\quad + 3\,\mathbb{E}\!\left[\left|\frac{1}{N}\sum_{n=1}^N (R^H_g(\tau_{n,0}) - b)
- (V_g^{\boldsymbol{\theta}^{(t)}}(\rho) - b)\right|^2\right]
+ 3\,|V_g^{\boldsymbol{\theta}^{(t)}}(\rho) - b|^2 \\[1ex]
&\overset{(b)}{\le}
\frac{3}{N}\sum_{n=1}^N
\mathbb{E}\!\left[\bigl|\sigma_c^{-1}(p^{(t)}_{c,n}) - (R^H_g(\tau_{n,0}) - b)\bigr|^2\right]
+ 3\sigma_1(H,N)
+ |V_g^{\boldsymbol{\theta}^{(t)}}(\rho) - b|^2 \\[1ex]
&\overset{(c)}{\le}
\frac{6L^2 \log M}{M}
+ \frac{12G(H)^2}{M^2}
+ 3\sigma_1(H,N)
+ \frac{3}{(1 - \gamma)^2}.
\end{align*}

In (a), the Cauchy--Schwarz inequality is used.  
Step (b) follows from the Cauchy--Schwarz inequality together with Lemma~\ref{Lemma:var of fh} is used.  
In (c), we use Lemma~\ref{Lemma:sigmac} along with the bound $\lvert V_g^{\boldsymbol{\theta}^{(t)}}(\rho) - b\rvert \le \tfrac{1}{1 - \gamma}$.
\end{proof}

We then establish a key Lemma, which will be used in the proofs of both theorems.
\begin{Lemma}
\label{Lemma:lambda}
The iterates $\{\lambda^{(t)}\}$ generated by line 23 of Algorithm~\ref{alg:NPGPD-HF} or line 17 of  Algorithm ~\ref{alg:ZPGPD-HF} satisfy
\begin{equation*}
\begin{aligned}
-\frac{1}{T}\sum_{t=0}^{T-1}\mathbb{E}\left[\lambda^{(t)}\bigl(V_g^{\pi^\star}(\rho) - V_g^{(t)}(\rho)\bigr)\right]\leq\bar{\sigma}_\lambda\epsilon_c(H,M)+\frac{\sigma^2_2(H,M,N)\eta_2}{2}.
\end{aligned}
\end{equation*}
\end{Lemma}
\begin{proof}
By line 23 of Algorithm \ref{alg:NPGPD-HF} or line 17 of Algorithm \ref{alg:ZPGPD-HF} , we get
\begin{align*}
0 &\leq \left( \lambda^{(T)} \right)^2 
= \sum_{t=0}^{T-1} \left( \left( \lambda^{(t+1)} \right)^2 - \left( \lambda^{(t)} \right)^2 \right) \\
&= \sum_{t=0}^{T-1} \left( \left( \mathcal{P}_\Lambda \left(  \lambda^{(t)} - \eta_2  \hat{h}_c^{(t)}(\rho) \right)  \right)^2 - \left( \lambda^{(t)} \right)^2 \right) \\
&\leq \sum_{t=0}^{T-1} \left( \left( \lambda^{(t)} - \eta_2 \hat{h}_c^{(t)}(\rho) \right)^2 - \left( \lambda^{(t)} \right)^2 \right) \\
&= -2\eta_2 \sum_{t=0}^{T-1} \lambda^{(t)} \hat{h}_c^{(t)}(\rho)
+ \eta_2^2 \sum_{t=0}^{T-1} (\hat{h}_c^{(t)}(\rho))^2 \\
&\leq 2\eta_2 \sum_{t=0}^{T-1} \lambda^{(t)} \left( V_g^\star(\rho) - V_g^{(t)}(\rho) \right)
+ 2\eta_2 \sum_{t=0}^{T-1} \lambda^{(t)} \left( V_g^{(t)}(\rho)-b - \hat{h}_c^{(t)} \right) \\
&\quad + \eta_2^2 \sum_{t=0}^{T-1} \left( \hat{h}_c^{(t)}(\rho)\right)^2,\\
\end{align*}
where the last inequality use the slater condition, i.e., $V_g^\star(\rho)-b\geq0$.\\
Take expectation of both sides, we get
\begin{equation*}
\begin{aligned}
&0\leq \sum_{t=0}^{T-1}\left(2\eta_2\mathbb{E}\left[ \lambda^{(t)} \left( V_g^{\pi^\star}(\rho) - V_g^{(t)}(\rho) \right)\right]
+ 2\eta_2 \mathbb{E}\left[ \lambda^{(t)} \left( V_g^{(t)}(\rho)-b - \hat{h}_c^{(t)}(\rho) \right)\right] \right.\\
&\left.\quad + \eta_2^2 \mathbb{E}\left[ \left( \hat{h}_c^{(t)}(\rho)\right)^2\right]\right).
\end{aligned}
\end{equation*}

By dividing both sides by $T$ and using Lemmas~\ref{Lemma:bias hc} and~\ref{Lemma:var hc}, together with the boundedness of $\lambda$, the proof is completed.
\end{proof}

\section{Proof of Theorem \ref{thm2}}
We first prove Theorem \ref{thm2}, as the proof of Theorem \ref{thm1} relies on several key concepts established in this section. The proof is based on \citet{ding2022ACC} and incorporates both the bias and variance introduced by finite horizon approximation and human feedback.
\label{proof:thm2}

An \emph{occupancy measure} \(q^\pi\) associated with a policy \(\pi\) is defined as  
\begin{align*}\label{eq:occ}
q^\pi_{s,a}
\;=\;\sum_{t=0}^{\infty}\gamma^t\,\Pr(s_t=s,\;a_t=a\mid \pi,\;s_0\sim\rho),
\end{align*}
for all \(s\in S\), \(a\in A\). For convenience, we collect these quantities into the vector form  
\[
\boldsymbol{q}^\pi \;=\;\bigl[q^\pi_{1,a},\dots,q^\pi_{|S|,a}\bigr]^\top\in\mathbb{R}^{|S|\!|\!A|},
\qquad
\boldsymbol{q}^\pi_a \;=\;\bigl[q^\pi_{s,a}\bigr]_{s\in S}\in\mathbb{R}^{|S|}.
\]

For each action \(a\), let \(P_a\in\mathbb{R}^{|S|\times|S|}\) denote the transition matrix with entries \(\{P(s' \mid s,a)\}_{s,s'}\). The set of feasible occupancy measures is then given by  
\[
Q \;\coloneqq \;\Bigl\{\,\boldsymbol{q}^\pi\in\mathbb{R}^{|S|\!|\!A|}\;\Bigm|\;
\sum_{a\in A}(I-\gamma P_a)\,\boldsymbol{q}^\pi_a = \rho,\ 
\boldsymbol{q}^\pi\ge0 \Bigr\}.
\]

With an abuse of notation, we treat  
\(
\boldsymbol{r}\in[0,1]^{|S|\!|\!A|},
\boldsymbol{g}\in[0,1]^{|S|\!|\!A|}
\)  
as vectors. The corresponding value functions under the initial distribution \(\rho\) can then be expressed as  
\[
V_r^\pi(\rho)=\langle \boldsymbol{q}^\pi,\boldsymbol{r}\rangle \coloneqq F_r(\boldsymbol{q}^\pi),
\qquad
V_g^\pi(\rho)=\langle \boldsymbol{q}^\pi,\boldsymbol{g}\rangle \coloneqq F_g(\boldsymbol{q}^\pi).
\]

We can now formulate the primal problem as a linear program:  
\begin{equation*}\label{eq:lp}
\max_{\,\boldsymbol{q}^\pi\in Q}\;F_r(\boldsymbol{q}^\pi)
\quad\text{subject to}\quad
F_g(\boldsymbol{q}^\pi)\;\ge\; b.
\end{equation*}

Once an optimal occupancy measure \(q^\pi\) is obtained, the corresponding policy can be recovered via  
\begin{equation*}\label{eq:recover}
\pi(a\mid s)
\;=\;\frac{q^\pi_{s,a}}{\sum_{a'\in A}q^\pi_{s,a'}}
\quad\forall s\in S,\;a\in A.
\end{equation*}

Let $\pi^{\boldsymbol{q}} : \mathcal{Q} \to \Delta_{\mathcal{A}}^{|\mathcal{S}|}$ denote a mapping from an occupancy measure $\boldsymbol{q}^{\pi}$ to a policy $\pi$.
Similarly, as defined in (48), let $\boldsymbol{q}^{\pi} : \Delta_{\mathcal{A}}^{|\mathcal{S}|} \to \mathcal{Q}$ denote a mapping from a policy $\pi$ to its corresponding occupancy measure $\boldsymbol{q}^{\pi}$.
Clearly, $\boldsymbol{q}^{\pi} = (\pi^{\boldsymbol{q}})^{-1}$.

We then reformulate the Lemma 2 of \citep{ding2022ACC} by incorporate bias and variance caused by finite horizon approximation and human feedback.
\begin{Lemma}
\label{Lemma:FLambda}
Let assumptions 1-3 hold. Let $\eta_1=\frac{1}{L_{\theta}}$ and $\eta_2=\frac{(1-\gamma)^2D_\theta L_\theta}{2\sqrt{T}}$. Fix \(
\Lambda = \bigl[0,\tfrac{2}{(1-\gamma)\,\xi}\bigr], 
\rho>0,\lambda^{(0)}=0,
{\boldsymbol{\theta}}^{(0)}\in{\boldsymbol{\Theta}},  d=|S||A|
\). The iterates $\boldsymbol{\theta}^{(t)}$ and $\lambda^{(t)}$ generated by Algorithm \ref{alg:ZPGPD-HF} satisfy
\begin{equation*}
\begin{aligned}
&\frac{1}{T}\sum_{t=0}^{T-1}\mathbb{E}\left[\left(F_r + \lambda^{(t)} F_g\right)(\boldsymbol{q}^{{\boldsymbol{\theta}}^\star}) - \left(F_r + \lambda^{(t)} F_g\right)(\boldsymbol{q}^{{\boldsymbol{\theta}}^{(t)}})\right]\\
&\leq\frac{2D_\theta L_\theta}{T^{1/4}}+(1+\frac{2}{(1-\gamma)\xi})\sqrt{D_\theta}\epsilon_{rg}(H,M)
\end{aligned}
\end{equation*}
where
\[
D_\theta := \frac{8|S|}{(1-\gamma)^2} \left\lVert \frac{d^{\pi^\star}_\rho}{\rho} \right\rVert_\infty^2
\quad\text{and}\quad
L_\theta := \frac{2|A|(1 + 2/\xi)}{(1-\gamma)^4}.
\]
\end{Lemma}

\begin{proof}
By Lemma~\ref{Lemma:smoothness}, we have
\[
\left| F_r(\boldsymbol{q}^{\boldsymbol{\theta}^\star}) - F_r(\boldsymbol{q}^{\boldsymbol{\theta}^{(t)}}) - \left\langle \nabla_\theta F_r(\boldsymbol{q}^{\boldsymbol{\theta}^{(t)}}), \boldsymbol{\theta} - \boldsymbol{\theta}^{(t)} \right\rangle \right|
\le \frac{\gamma |A|}{(1-\gamma)^3} \left\| 
\boldsymbol{\theta} - \boldsymbol{\theta}^{(t)} \right\|^2 .
\]
If we fix $\lambda^{(t)}\geq0$, then
\begin{equation}
\label{eq:FLip}
\begin{aligned}
(F_r + \lambda^{(t)} F_g)(\boldsymbol{q}^{\boldsymbol{\theta}}) 
&\geq (F_r + \lambda^{(t)} F_g)(\boldsymbol{q}^{{\boldsymbol{\theta}}^{(t)}}) 
+ \left\langle \nabla_{\boldsymbol{\theta}} F_r(\boldsymbol{q}^{{\boldsymbol{\theta}}^{(t)}}) 
+ \lambda^{(t)} \nabla_{\boldsymbol{\theta}} F_g(\boldsymbol{q}^{{\boldsymbol{\theta}}^{(t)}}),\, {\boldsymbol{\theta}} - {\boldsymbol{\theta}}^{(t)} \right\rangle \\
&\quad - \frac{L_{\boldsymbol{\theta}}}{2} \left\| {\boldsymbol{\theta}} - {\boldsymbol{\theta}}^{(t)} \right\|^2 \\
&\geq (F_r + \lambda^{(t)} F_g)(\boldsymbol{q}^{\boldsymbol{\theta}}) 
- L_{\boldsymbol{\theta}} \left\| {\boldsymbol{\theta}} - {\boldsymbol{\theta}}^{(t)} \right\|^2.
\end{aligned}
\end{equation}

From line 16 of Algorithm \ref{alg:ZPGPD-HF}, we know that the update of $\boldsymbol{\theta}^{(t)}$ is equivalent to
\begin{equation}
\label{eq:update theta}
\begin{aligned}
{\boldsymbol{\theta}}^{(t+1)} 
&= \arg\max_{{\boldsymbol{\theta}} \in {\Theta}} \left\{
V_r^{{\boldsymbol{\theta}}^{(t)}}(\rho) + \lambda^{(t)} V_g^{{\boldsymbol{\theta}}^{(t)}}(\rho)\right.\\
&\quad\left.+\left\langle \hat{\boldsymbol{h}}^{(t)}_r(\rho)+\lambda^{(t)}\hat{\boldsymbol{h}}^{(t)}_g(\rho),{\boldsymbol{\theta}} - {\boldsymbol{\theta}}^{(t)}\right\rangle
- \frac{1}{2\eta_1} \left\| {\boldsymbol{\theta}} - {\boldsymbol{\theta}}^{(t)} \right\|^2
\right\}\\
&= \arg\max_{{\boldsymbol{\theta}} \in {\Theta}} \left\{
V_r^{{\boldsymbol{\theta}}^{(t)}}(\rho) + \lambda^{(t)} V_g^{{\boldsymbol{\theta}}^{(t)}}(\rho)
+\left\langle \nabla_{\boldsymbol{\theta}} V_r^{{\boldsymbol{\theta}}^{(t)}}(\rho) + \lambda^{(t)} \nabla_{\boldsymbol{\theta}} V_g^{{\boldsymbol{\theta}}^{(t)}}(\rho), {\boldsymbol{\theta}} - {\boldsymbol{\theta}}^{(t)} \right\rangle \right.\\
&\left.\quad- \frac{1}{2\eta_1} \left\| {\boldsymbol{\theta}} - {\boldsymbol{\theta}}^{(t)} \right\|^2\right.\\
&\quad\left.+\left\langle\hat{\boldsymbol{h}}^{(t)}_r(\rho)+\lambda^{(t)}\hat{\boldsymbol{h}}^{(t)}_g(\rho)-(\nabla_{\boldsymbol{\theta}} V_r^{{\boldsymbol{\theta}}^{(t)}}(\rho) + \lambda^{(t)} \nabla_{\boldsymbol{\theta}} V_g^{{\boldsymbol{\theta}}^{(t)}}(\rho)),{\boldsymbol{\theta}} - {\boldsymbol{\theta}}^{(t)}\right\rangle
\right\}
\end{aligned}
\end{equation}
By taking $\eta_1 = 1 / L_{\theta}$ and $\boldsymbol{\theta} = \boldsymbol{\theta}^{(t+1)}$ in Equation (\ref{eq:update theta}),
\begin{equation}
\label{eq:F1}
\begin{aligned}
&(F_r + \lambda^{(t)} F_g)(\boldsymbol{q}^{{\boldsymbol{\theta}}^{(t+1)}})\\
&\geq \max_{{\boldsymbol{\theta}} \in {\Theta}} \left\{
(F_r + \lambda^{(t)} F_g)(\boldsymbol{q}^{{\boldsymbol{\theta}}^{(t)}})+ \left\langle \nabla_{\boldsymbol{\theta}} F_r(\boldsymbol{q}^{{\boldsymbol{\theta}}^{(t)}}) 
+ \lambda^{(t)} \nabla_{\boldsymbol{\theta}} F_g(\boldsymbol{q}^{{\boldsymbol{\theta}}^{(t)}}),\, {\boldsymbol{\theta}} - {\boldsymbol{\theta}}^{(t)} \right\rangle \right. \\
&\quad\left. 
- \frac{L_{\boldsymbol{\theta}}}{2} \left\| {\boldsymbol{\theta}} - {\boldsymbol{\theta}}^{(t)} \right\|^2+ \left\langle\hat{\boldsymbol{h}}^{(t)}_r+\lambda^{(t)}\hat{\boldsymbol{h}}^{(t)}_g-(\nabla_{\boldsymbol{\theta}} V_r^{{\boldsymbol{\theta}}^{(t)}}(\rho) + \lambda^{(t)} \nabla_{\boldsymbol{\theta}} V_g^{{\boldsymbol{\theta}}^{(t)}}(\rho)),{\boldsymbol{\theta}} - {\boldsymbol{\theta}}^{(t)}\right\rangle 
\right\} \\
&\overset{(a)}{\geq} \max_{{\boldsymbol{\theta}} \in {\Theta}} \left\{
(F_r + \lambda^{(t)} F_g)(\boldsymbol{q}^{\boldsymbol{\theta}})
- L_{\boldsymbol{\theta}} \left\| {\boldsymbol{\theta}} - {\boldsymbol{\theta}}^{(t)} \right\|^2\right.\\
&\left.\quad+\left\langle\hat{\boldsymbol{h}}^{(t)}_r+\lambda^{(t)}\hat{\boldsymbol{h}}^{(t)}_g-(\nabla_{\boldsymbol{\theta}} V_r^{{\boldsymbol{\theta}}^{(t)}}(\rho) + \lambda^{(t)} \nabla_{\boldsymbol{\theta}} V_g^{{\boldsymbol{\theta}}^{(t)}}(\rho)),{\boldsymbol{\theta}} - {\boldsymbol{\theta}}^{(t)}\right\rangle
\right\} \\
&\overset{(b)}{\geq} \max_{\alpha \in [0,1]} \left\{
(F_r + \lambda^{(t)} F_g)(\boldsymbol{q}^{{\boldsymbol{\theta}}^\alpha})
- L_{\boldsymbol{\theta}} \left\| {\boldsymbol{\theta}}^\alpha - {\boldsymbol{\theta}}^{(t)} \right\|^2\right.\\
&\left.\quad+\left\langle\hat{\boldsymbol{h}}^{(t)}_r+\lambda^{(t)}\hat{\boldsymbol{h}}^{(t)}_g-(\nabla_{\boldsymbol{\theta}} V_r^{{\boldsymbol{\theta}}^{(t)}}(\rho) + \lambda^{(t)} \nabla_{\boldsymbol{\theta}} V_g^{{\boldsymbol{\theta}}^{(t)}}(\rho)),{\boldsymbol{\theta}}^\alpha - {\boldsymbol{\theta}}^{(t)}\right\rangle
\right\},
\end{aligned}
\end{equation}
where $\boldsymbol{\theta}^{\alpha} := \pi^{(\alpha \boldsymbol{q}^{\boldsymbol{\theta}^*} + (1 - \alpha) \boldsymbol{q}^{\boldsymbol{\theta}^{(t)}})}$. 

In $(a)$, we applied Equation~(\ref{eq:FLip}). Step $(b)$ follows from $\pi^{\boldsymbol{q}} \circ \boldsymbol{q}^\pi = \mathrm{id}_{SA}$ and linearity of $\boldsymbol{q}^{\boldsymbol{\theta}}$ in $\boldsymbol{\theta}$. 

Since $F_r$ and $F_g$ are linear in $q^{\boldsymbol{\theta}}$, we have
\begin{equation}
\begin{aligned}
\left(F_r + \lambda^{(t)} F_g\right)(\boldsymbol{q}^{{\boldsymbol{\theta}}^\alpha}) 
= \alpha \left(F_r + \lambda^{(t)} F_g\right)(\boldsymbol{q}^{{\boldsymbol{\theta}}^\star}) 
+ (1 - \alpha) \left(F_r + \lambda^{(t)} F_g\right)(\boldsymbol{q}^{{\boldsymbol{\theta}}^{(t)}}).
\end{aligned}
\end{equation}
And from Equation (55) of \citep{ding2022convergence}, we have that
\begin{equation}
\label{eq:thetabound}
\left\| {\boldsymbol{\theta}}_\alpha - {\boldsymbol{\theta}}^{(t)} \right\|^2\leq\alpha^2D_\theta.
\end{equation}
Therefore, we have
\begin{equation*}
\begin{aligned}
&\left(F_r + \lambda^{(t)} F_g\right)(\boldsymbol{q}^{{\boldsymbol{\theta}}^\star}) 
- \left(F_r + \lambda^{(t)} F_g\right)(\boldsymbol{q}^{{\boldsymbol{\theta}}^{(t+1)}}) \\
&\overset{}{\leq} \min_{\alpha \in [0,1]} \left\{
L_{\boldsymbol{\theta}} \left\| {\boldsymbol{\theta}}^\alpha - {\boldsymbol{\theta}}^{(t)} \right\|^2 
+ \left(F_r + \lambda^{(t)} F_g\right)(\boldsymbol{q}^{{\boldsymbol{\theta}}^\star}) 
- \left(F_r + \lambda^{(t)} F_g\right)(\boldsymbol{q}^{{\boldsymbol{\theta}}^\alpha})\right.\\
&\left.\quad+\left\langle\hat{\boldsymbol{h}}^{(t)}_r+\lambda^{(t)}\hat{\boldsymbol{h}}^{(t)}_g-(\nabla_{\boldsymbol{\theta}} V_r^{{\boldsymbol{\theta}}^{(t)}}(\rho) + \lambda^{(t)} \nabla_{\boldsymbol{\theta}} V_g^{{\boldsymbol{\theta}}^{(t)}}(\rho)),{\boldsymbol{\theta}}^\alpha - {\boldsymbol{\theta}}^{(t)}\right\rangle \right\} \\
&\leq \min_{\alpha \in [0,1]} \left\{
\alpha^2 D_{\boldsymbol{\theta}} L_{\boldsymbol{\theta}}+(1 - \alpha) \left( \left(F_r + \lambda^{(t)} F_g\right)(\boldsymbol{q}^{{\boldsymbol{\theta}}^\star}) 
- \left(F_r + \lambda^{(t)} F_g\right)(\boldsymbol{q}^{{\boldsymbol{\theta}}^{(t)}}) \right)\right.\\
&\quad\left.+\left\langle\hat{\boldsymbol{h}}^{(t)}_r+\lambda^{(t)}\hat{\boldsymbol{h}}^{(t)}_g-(\nabla_{\boldsymbol{\theta}} V_r^{{\boldsymbol{\theta}}^{(t)}}(\rho) + \lambda^{(t)} \nabla_{\boldsymbol{\theta}} V_g^{{\boldsymbol{\theta}}^{(t)}}(\rho)),{\boldsymbol{\theta}}^\alpha - {\boldsymbol{\theta}}^{(t)}\right\rangle
\right\}
\end{aligned}
\end{equation*}
We can further get
\begin{equation*}
\begin{aligned}
&\left(F_r + \lambda^{(t+1)} F_g\right)(\boldsymbol{q}^{{\boldsymbol{\theta}}^\star}) 
- \left(F_r + \lambda^{(t+1)} F_g\right)(\boldsymbol{q}^{{\boldsymbol{\theta}}^{(t+1)}}) \\
&\leq \min_{\alpha \in [0,1]} \left\{
\alpha^2 D_{\boldsymbol{\theta}} L_{\boldsymbol{\theta}}
+ (1 - \alpha) \left( 
\left(F_r + \lambda^{(t)} F_g\right)(\boldsymbol{q}^{{\boldsymbol{\theta}}^\star}) 
- \left(F_r + \lambda^{(t)} F_g\right)(\boldsymbol{q}^{{\boldsymbol{\theta}}^{(t)}})
\right) \right. \\
&\quad\left.
+\left\langle\hat{\boldsymbol{h}}^{(t)}_r+\lambda^{(t)}\hat{\boldsymbol{h}}^{(t)}_g-(\nabla_{\boldsymbol{\theta}} V_r^{{\boldsymbol{\theta}}^{(t)}}(\rho) + \lambda^{(t)} \nabla_{\boldsymbol{\theta}} V_g^{{\boldsymbol{\theta}}^{(t)}}(\rho)),{\boldsymbol{\theta}}^\alpha - {\boldsymbol{\theta}}^{(t)}\right\rangle\right\}\\
&\quad - \left(\lambda^{(t)} - \lambda^{(t+1)}\right) 
\left(F_g(\boldsymbol{q}^{{\boldsymbol{\theta}}^\star}) - F_g(\boldsymbol{q}^{{\boldsymbol{\theta}}^{(t+1)}})\right).
\end{aligned}
\end{equation*}

Take conditional expectation of the drift over the natural filtration $\mathcal{F}_t$ we got
\begin{equation}
\label{eq:EF}
\begin{aligned}
&\mathbb{E}\left[\left(F_r + \lambda^{(t+1)} F_g\right)(\boldsymbol{q}^{{\boldsymbol{\theta}}^\star}) 
- \left(F_r + \lambda^{(t+1)} F_g\right)(\boldsymbol{q}^{{\boldsymbol{\theta}}^{(t+1)}})|\mathcal{F}_t \right] \\
&\leq \min_{\alpha \in [0,1]} \left\{
\alpha^2 D_{\boldsymbol{\theta}} L_{\boldsymbol{\theta}}
+ (1 - \alpha) \mathbb{E}\left[
\left(F_r + \lambda^{(t)} F_g\right)(\boldsymbol{q}^{{\boldsymbol{\theta}}^\star}) 
- \left(F_r + \lambda^{(t)} F_g\right)(\boldsymbol{q}^{{\boldsymbol{\theta}}^{(t)}})|\mathcal{F}_t
\right] \right. \\
&\quad\left.
+\mathbb{E}\left[\left\langle\hat{\boldsymbol{h}}^{(t)}_r(\rho)+\lambda^{(t)}\hat{\boldsymbol{h}}^{(t)}_g(\rho)-(\nabla_{\boldsymbol{\theta}} V_r^{{\boldsymbol{\theta}}^{(t)}}(\rho) + \lambda^{(t)} \nabla_{\boldsymbol{\theta}} V_g^{{\boldsymbol{\theta}}^{(t)}}(\rho)),{\boldsymbol{\theta}}^\alpha - {\boldsymbol{\theta}}^{(t)}\right\rangle|\mathcal{F}_t\right]\right\}\\
&\quad- \mathbb{E}\left[\left(\lambda^{(t)} - \lambda^{(t+1)}\right) 
\left(F_g(\boldsymbol{q}^{{\boldsymbol{\theta}}^\star})F_g(\boldsymbol{q}^{{\boldsymbol{\theta}}^{(t+1)}})\right)|\mathcal{F}_t\right]\\
&\overset{(a)}{\leq}\min_{\alpha \in [0,1]} \left\{
\alpha^2 D_{\boldsymbol{\theta}} L_{\boldsymbol{\theta}}
+ (1 - \alpha) \mathbb{E}\left[
\left(F_r + \lambda^{(t)} F_g\right)(\boldsymbol{q}^{{\boldsymbol{\theta}}^\star}) 
- \left(F_r + \lambda^{(t)} F_g\right)(\boldsymbol{q}^{{\boldsymbol{\theta}}^{(t)}})|\mathcal{F}_t
\right] \right. \\
&\quad\left.
+\left\|\mathbb{E}\left[\hat{\boldsymbol{h}}^{(t)}_r(\rho)+\lambda^{(t)}\hat{\boldsymbol{h}}^{(t)}_g(\rho)-(\nabla_{\boldsymbol{\theta}} V_r^{{\boldsymbol{\theta}}^{(t)}}(\rho) + \lambda^{(t)} \nabla_{\boldsymbol{\theta}} V_g^{{\boldsymbol{\theta}}^{(t)}}(\rho))|\mathcal{F}_t\right]\right\| \left\|{\boldsymbol{\theta}}^\alpha - {\boldsymbol{\theta}}^{(t)}\right\|\right\}+ \frac{D_\theta L_\theta}{2\sqrt{T}}\\
&\overset{(b)}{\leq}\min_{\alpha \in [0,1]} \left\{
\alpha^2 D_{\boldsymbol{\theta}} L_{\boldsymbol{\theta}}
+ (1 - \alpha) \mathbb{E}\left[
\left(F_r + \lambda^{(t)} F_g\right)(\boldsymbol{q}^{{\boldsymbol{\theta}}^\star}) 
- \left(F_r + \lambda^{(t)} F_g\right)(\boldsymbol{q}^{{\boldsymbol{\theta}}^{(t)}})|\mathcal{F}_t
\right]+ \frac{D_\theta L_\theta}{2\sqrt{T}} \right. \\
&\quad\left.
+\left(\left\|\mathbb{E}\left[\hat{\boldsymbol{h}}^{(t)}_r(\rho)-\nabla_{\boldsymbol{\theta}} V_r^{{\boldsymbol{\theta}}^{(t)}}(\rho) |\mathcal{F}_t\right]\right\| +\bar{\sigma}_\lambda \left\|\mathbb{E}\left[\hat{\boldsymbol{h}}^{(t)}_g(\rho)- \nabla_{\boldsymbol{\theta}}V_g^{{\boldsymbol{\theta}}^{(t)}}(\rho)|\mathcal{F}_t\right]\right\| \right)\left\|{\boldsymbol{\theta}}^\alpha - {\boldsymbol{\theta}}^{(t)}\right\|\right\}\\
&\overset{(c)}{\leq}\min_{\alpha \in [0,1]} \left\{
\alpha^2 D_{\boldsymbol{\theta}} L_{\boldsymbol{\theta}}
+ (1 - \alpha) \mathbb{E}\left[
\left(F_r + \lambda^{(t)} F_g\right)(\boldsymbol{q}^{{\boldsymbol{\theta}}^\star}) 
- \left(F_r + \lambda^{(t)} F_g\right)(\boldsymbol{q}^{{\boldsymbol{\theta}}^{(t)}})|\mathcal{F}_t
\right]\right.\\
&\left.\quad+\alpha\underbrace{(1+\bar{\sigma}_\lambda)\epsilon_{rg}(H,M)\sqrt{D_\theta}}_{\triangleq\epsilon_1}  \right\}+ \frac{D_\theta L_\theta}{2\sqrt{T}}.
\end{aligned}
\end{equation}
In $(a)$, we use Cauchy–Schwarz inequality and
\begin{equation*}
- \bigl(\lambda^{(t)} - \lambda^{(t+1)}\bigr)
\bigl(F_g(q^{{\boldsymbol{\theta}}^*}) - F_g(q^{{\boldsymbol{\theta}}^{(t+1)}})\bigr)
\;\le\;
\frac{\lvert \lambda^{(t)} - \lambda^{(t+1)}\rvert}{1 - \gamma}
\;\le\;
\frac{\eta_2}{(1 - \gamma)^2}=\frac{D_{\theta}L_{\theta}}{2\sqrt{T}}.
\end{equation*}
Step $(b)$ follows from Cauchy–Schwarz inequality and boundedness of $\lambda$.
In $(c)$, we use Equation (\ref{eq:thetabound}) and Lemma \ref{Lemma:overall}.

Let
\begin{equation}
\label{eq:beta}
\begin{aligned}
 \beta^{(t)}=\frac{\mathbb{E}\left[\left(F_r + \lambda^{(t)} F_g\right)(\boldsymbol{q}^{{\boldsymbol{\theta}}^\star}) 
- \left(F_r + \lambda^{(t)} F_g\right)(\boldsymbol{q}^{{\boldsymbol{\theta}}^{(t)}})|\mathcal{F}_t\right]-\epsilon_1
}{2D_{\theta}L_{\theta}}.
\end{aligned}
\end{equation}

When $\beta^{(t)}<0$, in order to minimize Equation (\ref{eq:EF}), we should set $\alpha=0$, we got:
\begin{equation*}
\begin{aligned}
&\mathbb{E}\left[\left(F_r + \lambda^{(t+1)} F_g\right)(\boldsymbol{q}^{{\boldsymbol{\theta}}^\star}) - \left(F_r + \lambda^{(t+1)} F_g\right)(\boldsymbol{q}^{{\boldsymbol{\theta}}^{(t+1)}})|\mathcal{F}_t\right]\\
&\leq\frac{D_\theta L_\theta}{2\sqrt{T}}+\epsilon_1.
\end{aligned}
\end{equation*}

When $\beta^{(t)}>1$, in order to minimize Equation (\ref{eq:EF}) ,we should set $\alpha=1$
that leads to
\begin{equation*}
\begin{aligned}
&\mathbb{E}\left[\left(F_r + \lambda^{(t+1)} F_g\right)(\boldsymbol{q}^{{\boldsymbol{\theta}}^\star}) - \left(F_r + \lambda^{(t+1)} F_g\right)(\boldsymbol{q}^{{\boldsymbol{\theta}}^{(t+1)}})|\mathcal{F}_t\right]\\
&\leq D_{\theta}L_{\theta}+\frac{D_{\theta}L_{\theta}}{2\sqrt{T}}\leq\frac{3}{2}D_{\theta}L_{\theta}.
\end{aligned}
\end{equation*}
While from Equation (\ref{eq:beta}) we can get $\beta^{(t)}<\frac{3}{4}$, which contradict.

When $0<\beta^{(t)}<1$, we can express Equation (\ref{eq:EF}) as
\begin{equation}
\begin{aligned}
\beta^{(t+1)}\leq(1-\frac{\beta^{(t)}}{2})\beta(t)+\frac{1}{4\sqrt{T}}.
\end{aligned}
\end{equation}
Without loss of generality, we assume that
\[
  0 \;\le\; \beta^{(t)} \;\le\;\frac{1}{T^{1/4}}\;\le\;1.
\]
By choosing $\lambda^{(0)} = 0$ and $\boldsymbol{\theta}^{(0)}$ such that 
$V_r^{\boldsymbol{\theta}^{(0)}}(\rho) \ge V_r^{\boldsymbol{\theta}^\star}(\rho)$, we know that 
$\beta^{(0)} \le 0$. Thus, $\beta^{(1)} \le \frac{1}{4\sqrt{T}}$. 
The case $\beta^{(t)} \le 0$ is trivial. 

Therefore, by induction of  over $t$, we have 
$\beta^{(t+1)} \le \frac{1}{T^{1/4}}$.
Using tower law, we can get
\begin{equation*}
\begin{aligned}
&\mathbb{E}\left[\left(F_r + \lambda^{(t)} F_g\right)(\boldsymbol{q}^{{\boldsymbol{\theta}}^\star}) 
- \left(F_r + \lambda^{(t)} F_g\right)(\boldsymbol{q}^{{\boldsymbol{\theta}}^{(t)}})\right]\\
&=\mathbb{E}\left[\mathbb{E}\left[\left(F_r + \lambda^{(t)} F_g\right)(\boldsymbol{q}^{{\boldsymbol{\theta}}^\star}) 
- \left(F_r + \lambda^{(t)} F_g\right)(\boldsymbol{q}^{{\boldsymbol{\theta}}^{(t)}})|\mathcal{F}_t\right]\right]\\
&\leq\frac{2D_\theta L_\theta}{T^{1/4}}+\epsilon_1.
\end{aligned}
\end{equation*}
Averaging over 
$t = 0, 1, \cdots, T-1$, we obtain
\begin{equation*}
\frac{1}{T}\sum_{t=0}^{T-1}\mathbb{E}\left[\left(F_r + \lambda^{(t)} F_g\right)(\boldsymbol{q}^{{\boldsymbol{\theta}}^\star}) 
- \left(F_r + \lambda^{(t)} F_g\right)(\boldsymbol{q}^{{\boldsymbol{\theta}}^{(t)}})\right]\leq\frac{2D_\theta L_\theta}{T^{1/4}}+\epsilon_1.
\end{equation*}
\end{proof}

\textbf{We now prove optimality gap in Theorem \ref{thm2}:}
Recall that corresponding value functions under the initial distribution \(\rho\) can then be expressed as  
\[
V_r^\pi(\rho)=\langle \boldsymbol{q}^\pi,\boldsymbol{r}\rangle \coloneqq F_r(\boldsymbol{q}^\pi),
\quad
V_g^\pi(\rho)=\langle \boldsymbol{q}^\pi,\boldsymbol{g}\rangle \coloneqq F_g(\boldsymbol{q}^\pi).
\]
Directly add Lemma \ref{Lemma:FLambda} and Lemma \ref{Lemma:lambda}, we get:
\begin{equation}
\begin{aligned}
&\frac{1}{T}\sum_{t=0}^{T-1}\mathbb{E}\left[V_r^\star(\rho) - V_r^{\boldsymbol{\theta}^{(t)}}(\rho)\bigr)\right]\\
&\leq\frac{2D_\theta L_\theta}{T^{1/4}}+(1+\frac{2}{(1-\gamma)\xi})\sqrt{D_\theta}\epsilon_{rg}(H,M)+\frac{2}{(1-\gamma)\xi}\epsilon_c(H,M)+\frac{\sigma^2_2(H,M,N)\eta_2}{2}
\end{aligned}
\end{equation}
By substituting $\epsilon_{rg}(H,M)$ and $\epsilon_c(H,M)$ with definition from Lemma \ref{Lemma:overall} and Lemma \ref{Lemma:bias hc}, we obtain:
\begin{equation*}
\begin{aligned}
&\frac{1}{T}\sum_{t=0}^{T-1}\mathbb{E}\left[V_r^\star(\rho) - V_r^{\boldsymbol{\theta}^{(t)}}(\rho)\right]\\
&\leq\frac{2D_\theta L_\theta}{T^{1/4}}+(1+\frac{2}{(1-\gamma)\xi})\sqrt{D_\theta}\left(\frac{\mu L'd}{2}+ \frac{2d\gamma^H}{\mu(1-\gamma)}+\frac{d}{\mu}\left( L \sqrt{ \frac{2 \log M}{M}} 
+ \frac{2G}{M^2}\right)\right)\\
&\quad+\frac{2}{(1-\gamma)\xi}\left(\frac{ \gamma^H}{1 - \gamma}+\left( L \sqrt{ \frac{2 \log M}{M} } + \frac{2G(H)}{M^2}\right)\right)+\frac{\sigma^2_2(H,M,N)\eta_2}{2}.
\end{aligned}
\end{equation*}
Let $\mu^2=\max\left\{\frac{2 \gamma^H}{(1-\gamma)},  L \sqrt{ \frac{2 \log M}{M} } 
+ \frac{2G(H)}{M^2}\right\}.$

We can get
\(
\frac{\mu L'd}{2}\geq\frac{2d\gamma^H}{\mu(1-\gamma)}
\)
and
\(
\frac{\mu L'd}{2}\geq\frac{d}{\mu}\left( L \sqrt{ \frac{2 \log M}{M} } 
+ \frac{2G}{M^2}\right).
\)
Thus,
\begin{equation*}
\begin{aligned}
&\frac{1}{T}\sum_{t=0}^{T-1}\mathbb{E}\left[V_r^\star(\rho) - V_r^{\boldsymbol{\theta}^{(t)}}(\rho)\right]\\
&\leq\frac{2D_\theta L_\theta}{T^{1/4}}+(1+\frac{2}{(1-\gamma)\xi})\sqrt{D_\theta}\left(\frac{3\mu L'd}{2}\right)\\
&\quad+\frac{2}{(1-\gamma)\xi}\left(\frac{ \gamma^H}{1 - \gamma}+\left( L \sqrt{ \frac{2 \log M}{M} } + \frac{2G(H)}{M^2}\right)\right)+\frac{\sigma^2_2(H,M,N)\eta_2}{2}.
\end{aligned}
\end{equation*}
By using the facts that 
\[
\frac{\sigma^2_2(H,M,N)\eta_2}{2} \sim \mathcal{O}\!\left(\frac{1}{\sqrt{T}}\right), \quad
\mu \leq \sqrt{\frac{2\gamma^H}{1-\gamma}} 
+ \sqrt{L\sqrt{\frac{2\log M}{M}} + \frac{2G(H)}{M^2}},
\]
\[
\mathcal{O}\!\left(\frac{1}{M^2}\right) \leq \mathcal{O}\!\left(\sqrt{\frac{\log M}{M}}\right),\quad
\text{and} \quad 
\sqrt{D_{\theta}} \leq \mathcal{O}(\sqrt{d}),
\]
we can derive the final result:
\begin{equation}
\label{th2opti}
\begin{aligned}
&\frac{1}{T}\sum_{t=0}^{T-1}\mathbb{E}\left[V_r^\star(\rho) - V_r^{\boldsymbol{\theta}^{(t)}}(\rho)\right]\leq\mathcal{O}\left(\frac{d}{T^{1/4}}+d^{3/2}  \sqrt{\frac{ \log M}{M}}
+d^{3/2}\sqrt{\frac{\gamma^H}{1-\gamma}}\right).
\end{aligned}
\end{equation}
\textbf{We then bound the constraint violation.}
By line 17 of Algorithm \ref{alg:ZPGPD-HF}, for any $\lambda \in \left[0, \frac{2}{(1-\gamma)\xi}\right]$, 
\begin{align*}
\left|\lambda^{(t+1)} - \lambda\right|^2 
&\le 
\left| \lambda^{(t)} - \eta_2 \left(\hat{h}_c^{(\rho)}\right) - \lambda \right|^2 \\
&= 
\left|\lambda^{(t)} - \lambda\right|^2 
- 2\eta_2 \left(\hat{h}_c^{(t)}(\rho)\right)\left(\lambda^{(t)} - \lambda\right) 
+ \eta_2^2 \left(\hat{h}_c^{(t)}(\rho)\right)^2, \\
\end{align*}
where the first inequality is due to the non-expansiveness of projection operator $\mathcal{P}_\Lambda$ and the last inequality is due to 
$\left(V_g^{\boldsymbol{\theta}^{(t)}}(\rho) - b\right)^2 \le \frac{1}{(1-\gamma)^2}$.
Summing it up from $t = 0$ to $t = T-1$ and dividing it by $T$ yield
\[
0 \le \frac{1}{T}\left|\lambda^{(T)} - \lambda\right|^2 
\le \frac{1}{T}\left|\lambda^{(0)} - \lambda\right|^2 
- \frac{2\eta_2}{T} \sum_{t=0}^{T-1} 
\left(\hat{h}_c^{(t)}(\rho) \right)\left(\lambda^{(t)} - \lambda\right) 
+ \eta_2^2(\hat{h}_c^{(t)}(\rho))^2.
\]
Take expectation and use Lemma \ref{Lemma:var hc}, we can get
\[
\frac{1}{T} \sum_{t=0}^{T-1} 
\mathbb{E}\left(\hat{h}_c^{(t)}(\rho) \right)\left(\lambda^{(t)} - \lambda\right)
\le 
\frac{1}{2\eta_2 T}\left|\lambda^{(0)} - \lambda\right|^2 
+ \frac{\eta_2}{2}\sigma_2^2(H,M,N).
\]
Add it to Lemma \ref{Lemma:FLambda}, we get
\begin{align*}
&\frac{1}{T} \sum_{t=0}^{T-1} \mathbb{E}\Big( V_{r}^{\pi^\star}(\rho) - V_{r}^{\boldsymbol{\theta}^{(t)}}(\rho) \Big)
+ \frac{1}{T} \sum_{t=0}^{T-1} \lambda^{(t)} \mathbb{E}\left( V_{g}^{\pi^\star}(\rho) - V_{g}^{\boldsymbol{\theta}^{(t)}}(\rho)+\hat{h}_c^{(t)}(\rho) \right)\\
&-\frac{\lambda}{T}\sum_{t=0}^{T-1}\mathbb{E}\left(\hat{h}_c^{(t)}(\rho)\right)\\
&=\frac{1}{T} \sum_{t=0}^{T-1} \mathbb{E}\Big( V_{r}^{\pi^\star}(\rho) - V_{r}^{\boldsymbol{\theta}^{(t)}}(\rho) \Big)\\
&\quad+ \frac{1}{T} \sum_{t=0}^{T-1} \lambda^{(t)} \mathbb{E}\left( V_{g}^{\pi^\star}(\rho) - V_{g}^{\boldsymbol{\theta}^{(t)}}(\rho)+V_{g}^{\boldsymbol{\theta}^{(t)}}(\rho)-b+\hat{h}_c^{(t)}(\rho)-\left(V_{g}^{\boldsymbol{\theta}^{(t)}}(\rho)-b\right) \right)\\
&\quad-\frac{\lambda}{T}\sum_{t=0}^{T-1}\mathbb{E}\left(\hat{h}_c^{(t)}(\rho)-\left(V_{g}^{\boldsymbol{\theta}^{(t)}}(\rho)-b\right)+\left(V_{g}^{\boldsymbol{\theta}^{(t)}}(\rho)-b\right)\right)\\
&\leq\frac{2D_\theta L_\theta}{T^{1/4}}+(1+\frac{2}{(1-\gamma)\xi})\sqrt{D_\theta}\epsilon_{rg}(H,M)+\frac{1}{2\eta_2 T}\left|\lambda^{(0)} - \lambda\right|^2 
+ \frac{\eta_2}{2}\sigma_2^2(H,M,N).
\end{align*}
By using Lemma \ref{Lemma:bias hc}, boundedness of $\lambda^{(t)}$ and the fact that $V_{g}^{\pi^\star}(\rho)-b>0$, we can get 
\begin{align*}
&\quad\frac{1}{T} \sum_{t=0}^{T-1} \mathbb{E}\left( V_{r}^{\pi^\star}(\rho) - V_{r}^{\boldsymbol{\theta}^{(t)}}(\rho) \right)- 2\bar{\sigma}_{\lambda} \epsilon_c(H,M)+\frac{\lambda}{T}\sum_{t=0}^{T-1}\mathbb{E}\left(b-V_{g}^{\boldsymbol{\theta}^{(t)}}(\rho)\right)\\
&\leq\frac{2D_\theta L_\theta}{T^{1/4}}+(1+\frac{2}{(1-\gamma)\xi})\sqrt{D_\theta}\epsilon_{rg}(H,M)+\frac{1}{2\eta_2 T}\left|\lambda^{(0)} - \lambda\right|^2 
+ \frac{\eta_2}{2}\sigma_2^2(H,M,N).
\end{align*}
We take $\lambda=\bar{\sigma}_{\lambda}$ when $\sum_{t=0}^{T-1}\mathbb{E}\left(b-V_{g}^{\boldsymbol{\theta}^{(t)}}(\rho)\right)\geq0$; otherwise $\lambda=0$.
Thus,
\begin{align*}
&V_{r}^{\pi^\star}(\rho)- \mathbb{E}\left( \frac{1}{T} \sum_{t=0}^{T-1} V_{r}^{(t)}(\rho) \right)- 2\bar{\sigma}_{\lambda} \epsilon_c(H,M)+\bar\sigma_\lambda\left[b-\mathbb{E}\left(\frac{1}{T}\sum_{t=0}^{T-1}V_{g}^{(t)}(\rho)\right)\right]_+\\
&\leq\frac{2D_\theta L_\theta}{T^{1/4}}+(1+\frac{2}{(1-\gamma)\xi})\sqrt{D_\theta}\epsilon_{rg}(H,M)+\frac{1}{2\eta_2 T}\left|\lambda^{(0)} - \lambda\right|^2 
+ \frac{\eta_2}{2}\sigma_2^2(H,M,N).
\end{align*}
As in the beginning of this section, we know that $V^{\boldsymbol{\theta}^{(t)}}_r(\rho)$ and $V^{\boldsymbol{\theta}^{(t)}}_g(\rho)$ are linear functions of the occupancy measure $q^{\pi}$ induced by policy $\pi^{(t)}$.
By the convexity of the set of occupancy measures, the average of $T$ occupancy measures and take expectation is still an occupancy measure that induces a policy $\pi'$ with values $V^{\pi'}_r$ and $V^{\pi'}_g$.
Therefore, there exists a policy $\pi'$ such that 
\[
V^{\pi'}_r(\rho) =\mathbb{E} \left(\frac{1}{T} \sum_{t=0}^{T-1} V^{\boldsymbol{\theta}^{(t)}}_r(\rho)\right), 
\quad 
V^{\pi'}_g(\rho) = \mathbb{E} \left(\frac{1}{T} \sum_{t=0}^{T-1} V^{\boldsymbol{\theta}^{(t)}}_g(\rho)\right).
\]

\medskip

Therefore,
\begin{align*}
&V_r^\star(\rho) - V^{\pi'}_r(\rho) + \bar{\sigma}_{\lambda} \big[ b - V^{\pi'}_g(\rho) \big]_+\\
&\leq\frac{2D_\theta L_\theta}{T^{1/4}}+(1+\frac{2}{(1-\gamma)\xi})\sqrt{D_\theta}\epsilon_{rg}(H,M)+\frac{1}{2\eta_2 T}\left|\lambda^{(0)} - \lambda\right|^2 
+ \frac{\eta_2}{2}\sigma_2^2(H,M,N).
\end{align*}

We note that $\bar\sigma_{\lambda}=\frac{2}{(1-\gamma)\xi} \ge 2\lambda^\star$. 
According to Lemma~\ref{Lemma:constraint_violation} and Lemma \ref{Lemma:overall}, we obtain
\begin{align*}
&\big[ b - V^{\pi'}_g(\rho) \big]_+\\
& 
\leq\frac{2}{\bar{\sigma}_{\lambda}}\left(\frac{2D_\theta L_\theta}{T^{1/4}}+(1+\frac{2}{(1-\gamma)\xi})\sqrt{D_\theta}\epsilon_{rg}(H,M)+\frac{1}{2\eta_2 T}\left|\lambda^{(0)} - \lambda\right|^2 
+ \frac{\eta_2}{2}\sigma_2^2(H,M,N)\right)\\
&\leq\frac{2}{\bar{\sigma}_{\lambda}}\left(\frac{2D_\theta L_\theta}{T^{1/4}}+(1+\frac{2}{(1-\gamma)\xi})\sqrt{D_\theta}\left(\frac{\mu L'd}{2}+ \frac{2d\gamma^H}{\mu(1-\gamma)}+\frac{d}{\mu}\left( L \sqrt{ \frac{2 \log M}{M}} 
+ \frac{2G(H)}{M^2}\right)\right)\right.\\
&\quad\left.+\frac{1}{2\eta_2 T}\left|\lambda^{(0)} - \lambda\right|^2 
+ \frac{\eta_2}{2}\sigma_2^2(H,M,N)\right).
\end{align*}
Similarly to the derivation of Equation~\ref{th2opti}, let
\[
\mu^2 = \max\!\left\{
\frac{2\gamma^H}{1-\gamma}, \;
L\sqrt{\frac{2\log M}{M}} + \frac{2G(H)}{M^2}
\right\},
\]
and use the facts that
\[
\frac{1}{2\eta_2 T}|\lambda^{(0)} - \lambda|^2 
\leq \mathcal{O}\!\left(\frac{1}{T^{1/4}}\right), 
\quad 
\frac{\eta_2}{2}\sigma_2^2(H,M,N)
\leq \mathcal{O}\!\left(\frac{1}{T^{1/4}}\right),
\]
\[
\mu \leq 
\sqrt{\frac{2\gamma^H}{1-\gamma}} 
+ \sqrt{L\sqrt{\frac{2\log M}{M}} + \frac{2G(H)}{M^2}},
\]
\[
\mathcal{O}\!\left(\frac{1}{M^2}\right) 
\leq \mathcal{O}\!\left(\sqrt{\frac{\log M}{M}}\right),
\quad \text{and} \quad 
\sqrt{D_{\theta}} \leq \mathcal{O}(\sqrt{d}),
\]
we obtain the final result:
\begin{equation*}
\frac{1}{T} \sum_{t=0}^{T-1} 
\mathbb{E}\!\left[ b - V_{g}^{\boldsymbol{\theta}^{(t)}}(\rho) \right]_+
\;\leq\; 
\mathcal{O}\!\left(\frac{d}{T^{1/4}}+d^{3/2}  \left(\frac{ \log M}{M} \right)^{1/4}
+d^{3/2}\sqrt{\frac{\gamma^H}{1-\gamma}}\right).
\end{equation*}

\section{Proof of Theorem \ref{thm1}}
\label{proof:thm1}
The proof follows the steps of \citep{ding2020natural}, with an additional step to account for bias in the advantage function.

We first provide the performance difference Lemma from \citep{kakade2002approximately}.
\begin{Lemma}
\label{Lemma:performance diff}
The performance difference Lemma quantifies the difference between $V^{\pi}_{\diamond}(s_{0})$ and
$V^{\pi'}_{\diamond}(s_{0})$ for any two policies $\pi$ and $\pi'$
and any state $s_{0}$,
\[
V^{\pi}_{\diamond}(s_{0}) - V^{\pi'}_{\diamond}(s_{0})
= \frac{1}{1-\gamma} \; \mathbb{E}_{s \sim d^{\pi}_{s_{0}}, \, a \sim \pi(\cdot|s)}
\left[ A^{\pi'}_{\diamond}(s,a) \right].
\]
\end{Lemma}
We then quantify the bias in the estimated advantage function induced by human feedback and finite horizon approximation.
\begin{Lemma}
\label{Lemma:bias advantage}
The bias between the advantage function estimated at line~19 of Algorithm~\ref{alg:NPGPD-HF} and the ground truth advantage function satisfies
\begin{align*}
\left\|
\mathbb{E}\!\left[\hat A_\diamond^{(t)}(s,a) - A_\diamond^{(t)}(s,a)\right]
\right\|
\le
\underbrace{
L\sqrt{\frac{2 \log M}{M}} + \frac{2G(H)}{M^2} + 2\epsilon_V(H)
}_{\triangleq\, \epsilon_A(M,H)}.
\end{align*}
\end{Lemma}

\begin{proof}
\begin{align*}
&\big\|
\mathbb{E}\!\left[\hat A_\diamond^{(t)}(s,a) - A_\diamond^{(t)}(s,a)\right]
\big\|\\
&\overset{(a)}{\le}
\frac{1}{N}\sum_{n=1}^{N}
\Big\|
\mathbb{E}\!\left[
\sigma_\diamond^{-1}(p^{(t)}_{\diamond,n}(s,a))
- \bigl(R_\diamond^H(\tau_{n}(s,a)) - R_\diamond^H(\tau_{n}(s))\bigr) \right.\\
&\quad\left.
+ \bigl(R_\diamond^H(\tau_{n}(s,a)) - R_\diamond^H(\tau_{n}(s))\bigr)
- \bigl(Q_\diamond^{\boldsymbol{\theta}^{(t)}}(s,a) - V_\diamond^{\boldsymbol{\theta}^{(t)}}(s)\bigr)
\right]
\Big\| \\[1ex]
&\overset{(b)}{\le}
\left\|
\mathbb{E}\!\left[
\sigma_\diamond^{-1}(p^{(t)}_{\diamond,n}(s,a))
- \bigl(R_\diamond^H(\tau_{n}(s,a)) - R_\diamond^H(\tau_{n}(s))\bigr)
\right]
\right\| \\[0.5ex]
&\quad
+ \left\|
\mathbb{E}\!\left[
R_\diamond^H(\tau_{n}(s,a))
- Q_\diamond^{\boldsymbol{\theta}^{(t)}}(s,a)
\right]
\right\|
+ \left\|
\mathbb{E}\!\left[
R_\diamond^H(\tau_{n}(s))
- V_\diamond^{\boldsymbol{\theta}^{(t)}}(s)
\right]
\right\| \\[1ex]
&\overset{(c)}{\le}
L\sqrt{\frac{2 \log M}{M}}
+ \frac{2G(H)}{M^2}
+ 2\epsilon_V(H).
\end{align*}

In (a), the triangle inequality is used.  
In (b), the triangle inequality is combined with the independence of sampled trajectories.  
Step (c) follows from Lemma~\ref{Lemma:human_feedback_bias_var}, Lemma~\ref{Lemma:RVRQ}, and the definitions of $Q_\diamond^{\boldsymbol{\theta}^{(t)}}(s,a)$ and $V_\diamond^{\boldsymbol{\theta}^{(t)}}(s)$.
\end{proof}
Taking into account the bias of the advantage function, we reformulate Lemma 6 of \citep{ding2020natural} as follows.
\begin{Lemma}
\label{Lemma:nonmonotone}
Suppose \[
\hat Z^{(t)}(s) := \sum_{a \in A} \pi^{(t)}(a \mid s) 
\exp\!\left(\tfrac{\eta_{1}}{1-\gamma} \hat A_{L}^{(t)}(s,a)\right),
\]
 the iterates $\pi_{\boldsymbol{\theta}^{(t)}}$ generated by line 22 of Algorithm \ref{alg:NPGPD-HF} satisfy
\begin{equation*}
\begin{aligned}
V_{r}^{\boldsymbol{\theta}^{(t+1)}}(\mu) - V_{r}^{\boldsymbol{\theta}^{(t)}}(\mu) + \lambda^{(t)} \left( V_{g}^{\boldsymbol{\theta}^{(t+1)}}(\mu) - V_{g}^{\boldsymbol{\theta}^{(t)}}(\mu) \right)+\frac{1+\bar{\sigma}_\lambda}{1-\gamma}\epsilon_A(M,H)\geq\frac{1-\gamma}{\eta_1} \; \mathbb{E}_{s \sim \mu} \log \hat Z^{(t)}(s).
\end{aligned}
\end{equation*}
\end{Lemma}
\begin{proof}
We begin by prove the following equality.
\begin{align}
\label{boundz}
\log \hat Z^{(t)}(s) 
&= \log \Bigg( \sum_{a \in \mathcal{A}} \pi^{(t)}(a \mid s) 
   \exp \Bigg( \frac{\eta_1}{1-\gamma} \Big(\hat A_r^{(t)}(s,a) + \lambda^{(t)} \hat A_g^{(t)}(s,a) \Big) \Bigg) \Bigg) \notag\\
&\overset{(a)}{\geq} \sum_{a \in \mathcal{A}} \pi^{(t)}(a \mid s) 
   \log \Bigg( \exp \Bigg( \frac{\eta_1}{1-\gamma} \Big( \hat A_r^{(t)}(s,a) + \lambda^{(t)} \hat A_g^{(t)}(s,a) \Big) \Bigg) \Bigg) \notag\\
&= \frac{\eta_1}{1-\gamma} \sum_{a \in \mathcal{A}} \pi^{(t)}(a \mid s) 
   \left( A_r^{(t)}(s,a) + \lambda^{(t)} A_g^{(t)}(s,a) \right)\notag\\
   &+ \frac{\eta_1}{1-\gamma} \sum_{a \in \mathcal{A}} \pi^{(t)}(a \mid s) \left( \hat A_r^{(t)}(s,a)-A_r^{(t)}(s,a) + \lambda^{(t)}(\hat A_g^{(t)}(s,a) - A_g^{(t)}(s,a)) \right) \notag\\
&\overset{(b)}{=} \frac{\eta_1}{1-\gamma} \sum_{a \in \mathcal{A}} \pi^{(t)}(a \mid s) A_r^{(t)}(s,a) 
   + \frac{\eta_1}{1-\gamma} \lambda^{(t)} \sum_{a \in \mathcal{A}} \pi^{(t)}(a \mid s) A_g^{(t)}(s,a)\\
   &-\frac{\eta_1(1+\lambda^{(t)} )}{1-\gamma}\epsilon_A(M,H) \notag\\
&\overset{(c)}{=} -\frac{\eta_1(1+\bar{\sigma}_\lambda )}{1-\gamma}\epsilon_A(M,H),
\end{align}
where $(a)$ follows from Jensen's inequality, $(b)$ follows from Lemma~\ref{Lemma:bias advantage} and boundedness of $\lambda$ and $(c)$ is true due to the definition of the Advantage function.
We can then have 
\begin{align*}
&V_{r}^{\boldsymbol{\theta}^{(t+1)}}(\mu) - V_{r}^{\boldsymbol{\theta}^{(t)}}(\mu)\\ 
&\overset{(a)}{=} \frac{1}{1-\gamma} \; \mathbb{E}_{s \sim d_{\mu}^{(t+1)},\, a \sim \pi^{(t+1)}(\cdot \mid s)}
   \left[ A_{r}^{(t)}(s,a) \right] \\
&\overset{(b)}{=} \frac{1}{1-\gamma} \; \mathbb{E}_{s \sim d_{\mu}^{(t+1)}}
   \left[ \sum_{a \in A} \pi^{(t+1)}(a \mid s) A_{r}^{(t)}(s,a) \right] \\
&\overset{(c)}{=} \frac{1}{1-\gamma} \; \mathbb{E}_{s \sim d_{\mu}^{(t+1)}}
   \left[ \sum_{a \in A} \pi^{(t+1)}(a \mid s) \hat A_{r}^{(t)}(s,a) \right]\\
&\quad+\frac{1}{1-\gamma} \; \mathbb{E}_{s \sim d_{\mu}^{(t+1)}}
   \left[ \sum_{a \in A} \pi^{(t+1)}(a \mid s) \left( A_{r}^{(t)}(s,a)-\hat A_{r}^{(t)}(s,a)\right) \right]\\
&\overset{(d)}{\geq} \frac{1}{\eta_1} \; \mathbb{E}_{s \sim d_{\mu}^{(t+1)}}
   \left[ \sum_{a \in A} \pi^{(t+1)}(a \mid s) 
   \log \Bigg( \frac{\pi^{(t+1)}(a \mid s)}{\pi^{(t)}(a \mid s)} \hat Z^{(t)}(s) \Bigg) \right] \\
\quad &- \frac{\lambda^{(t)}}{1-\gamma} \; \mathbb{E}_{s \sim d_{\mu}^{(t+1)}}
   \left[ \sum_{a \in A} \pi^{(t+1)}(a \mid s) \hat A_{g}^{(t)}(s,a) \right] -\frac{1}{1-\gamma}\epsilon_A(M,H)\\
&\overset{(e)}{=}\frac{1}{\eta_1} \; \mathbb{E}_{s \sim d_{\mu}^{(t+1)}}
   \left[ D_{\mathrm{KL}}\!\left( \pi^{(t+1)}(\cdot \mid s) \,\|\, \pi^{(t)}(\cdot \mid s) \right) \right]  + \frac{1}{\eta_1} \; \mathbb{E}_{s \sim d_{\mu}^{(t+1)}} \log \hat Z^{(t)}(s) \\
&\quad - \frac{\lambda^{(t)}}{1-\gamma} \; \mathbb{E}_{s \sim d_{\mu}^{(t+1)}}
   \left[ \sum_{a \in A} \pi^{(t+1)}(a \mid s) A_{g}^{(t)}(s,a) \right]\\
&\quad + \frac{\lambda^{(t)}}{1-\gamma} \; \mathbb{E}_{s \sim d_{\mu}^{(t+1)}}
   \left[ \sum_{a \in A} \pi^{(t+1)}(a \mid s) \left(A_{g}^{(t)}(s,a) -\hat A_{g}^{(t)}(s,a)\right)\right]-\frac{1}{1-\gamma}\epsilon_A(M,H) \\
&\overset{(f)}{\geq} \frac{1}{\eta_1} \; \mathbb{E}_{s \sim d_{\mu}^{(t+1)}} \log \hat Z^{(t)}(s)- \frac{\lambda^{(t)}}{1-\gamma} \; \mathbb{E}_{s \sim d_{\mu}^{(t+1)}}
   \left[ \sum_{a \in A} \pi^{(t+1)}(a \mid s) A_{g}^{(t)}(s,a) \right] \\
   &\quad-\frac{1+\bar{\sigma}_{\lambda}}{1-\gamma}\epsilon_A(M,H)\\
&\overset{(g)}{=} \frac{1}{\eta_1} \; \mathbb{E}_{s \sim d_{\mu}^{(t+1)}} \left(\log \hat Z^{(t)}(s)+\frac{\eta_1(1+\bar{\sigma}_\lambda )}{1-\gamma}\epsilon_A(M,H)\right)
  \\
&\quad -\frac{1}{\eta_1} \; \mathbb{E}_{s \sim d_{\mu}^{(t+1)}}\left(\frac{\eta_1(1+\bar{\sigma}_\lambda )}{1-\gamma}\epsilon_A(M,H)\right)\\
&\quad- \lambda^{(t)} \left( V_{g}^{(t+1)}(\mu) - V_{g}^{(t)}(\mu) \right)-\frac{1+\bar{\sigma}_{\lambda}}{1-\gamma}\epsilon_A(M,H)\\
&\overset{(h)}{\geq} \frac{1-\gamma}{\eta_1} \; \mathbb{E}_{s \sim \mu} \left(\log \hat Z^{(t)}(s)+\frac{\eta_1(1+\bar{\sigma}_\lambda )}{1-\gamma}\epsilon_A(M,H)\right)\\
&\quad-\frac{1}{\eta_1} \; \mathbb{E}_{s \sim d_{\mu}^{(t+1)}}\left(\frac{\eta_1(1+\bar{\sigma}_\lambda )}{1-\gamma}\epsilon_A(M,H)\right)\\
 &\quad  - \lambda^{(t)} \left( V_{g}^{(t+1)}(\mu) - V_{g}^{(t)}(\mu) \right)-\frac{1+\bar{\sigma}_{\lambda}}{1-\gamma}\epsilon_A(M,H).
\end{align*}
In (a), we use Lemma~\ref{Lemma:performance diff}.
In (b), the result follows directly from the definition.
In (d), Lemma~\ref{Lemma:derive advantage} and Lemma~\ref{Lemma:bias advantage} are used.
Step (e) follows form the definition of the KL divergence.
In (f), the non-negativity of the K–L divergence, Lemma~\ref{Lemma:bias advantage}, and the boundedness of $\lambda$ are used.
In (g), Lemma~\ref{Lemma:performance diff} and Equation~(\ref{boundz}) are employed.
Finally, in (h), we use the inequality $d^{(t+1)}_{\mu} \ge (1 - \gamma)\mu$, which completes the proof.
\end{proof}

We then reformulate Lemma~7 of \citet{ding2020natural} on the average difference from the optimal policy by incorporating the bias and variance introduced by finite horizon approximation and human feedback.

\begin{Lemma}
\label{Lemma:bounded_avg_perf}
Let Assumptions hold. 
Fix $T > 0$, $\rho \in \Delta_{\mathcal{S}}$, $\theta^{(0)} = 0$, and $\lambda^{(0)} = 0$. 
Then the iterates $\pi_{\boldsymbol{\theta}^{(t)}}$ and $\lambda^{(t)}$ generated by Algorithm \ref{alg:NPGPD-HF} satisfy
\begin{align*}
&\frac{1}{T} \sum_{t=0}^{T-1} \Big( V_{r}^{\pi^\star}(\rho) - V_{r}^{(t)}(\rho) \Big)
+ \frac{1}{T} \sum_{t=0}^{T-1} \lambda^{(t)} \Big( V_{g}^{\pi^\star}(\rho) - V_{g}^{(t)}(\rho) \Big)
\\
&\leq\frac{\log |A|}{\eta_1 T}
+ \frac{1}{(1-\gamma)^2 T}
+ \frac{2 \eta_2}{(1-\gamma)^3}+2\frac{1+\bar{\sigma}_\lambda}{1-\gamma}\epsilon_A(M,H)
\end{align*}
\end{Lemma}
\begin{proof}
We use $d^\star$ to denote $d_\rho^{\pi^\star}$, we can then have
\begin{align*}
&V_{r}^{\pi^\star}(\rho) - V_{r}^{(t)}(\rho)\\ 
&\overset{(a)}{=} \frac{1}{1-\gamma} \; \mathbb{E}_{s \sim d^\star} 
   \left[ \sum_{a \in A} \pi^\star(a \mid s) A_{r}^{(t)}(s,a) \right] \\
&\overset{(b)}{=} \frac{1}{\eta_1} \; \mathbb{E}_{s \sim d^\star} 
   \left[ \sum_{a \in A} \pi^\star(a \mid s) 
   \log \left( \frac{\pi^{(t+1)}(a \mid s)}{\pi^{(t)}(a \mid s)} \hat Z^{(t)}(s) \right) \right]\\
&\quad +\frac{1}{1-\gamma} \; \mathbb{E}_{s \sim d_{}^{\star}}
   \left[ \sum_{a \in A} \pi^{(t+1)}(a \mid s) \left( A_{r}^{(t)}(s,a)-\hat A_{r}^{(t)}(s,a)\right) \right] \\
&\quad - \frac{\lambda^{(t)}}{1-\gamma} \; \mathbb{E}_{s \sim d^\star} 
   \left[ \sum_{a \in A} \pi^\star(a \mid s) A_{g}^{(t)}(s,a) \right] \\
&\quad+\frac{\lambda^{(t)}}{1-\gamma} \; \mathbb{E}_{s \sim d_{\mu}^{(t+1)}}
   \left[ \sum_{a \in A} \pi^{(t+1)}(a \mid s) \left( A_{g}^{(t)}(s,a)-\hat A_{g}^{(t)}(s,a)\right) \right]\\
&\overset{(c)}{\leq} \frac{1}{\eta_1} \; \mathbb{E}_{s \sim d^\star} 
   \left[ D_{\mathrm{KL}}\!\left( \pi^\star(\cdot \mid s) \,\|\, \pi^{(t)}(\cdot \mid s) \right)
        - D_{\mathrm{KL}}\!\left( \pi^\star(\cdot \mid s) \,\|\, \pi^{(t+1)}(\cdot \mid s) \right) \right] \\
&\quad + \frac{1}{\eta_1} \; \mathbb{E}_{s \sim d^\star} \log \hat Z^{(t)}(s)
   - \frac{\lambda^{(t)}}{1-\gamma} \; \mathbb{E}_{s \sim d^\star} 
   \left[ \sum_{a \in A} \pi^\star(a \mid s) A_{g}^{(t)}(s,a) \right]+\frac{1+\lambda^{(t)}}{1-\gamma}\epsilon_A(M,H) \\
&\overset{(d)}{=} \frac{1}{\eta_1} \; \mathbb{E}_{s \sim d^\star} 
   \left[ D_{\mathrm{KL}}\!\left( \pi^\star(\cdot \mid s) \,\|\, \pi^{(t)}(\cdot \mid s) \right)
        - D_{\mathrm{KL}}\!\left( \pi^\star(\cdot \mid s) \,\|\, \pi^{(t+1)}(\cdot \mid s) \right) \right] \\
&\quad + \frac{1}{\eta_1} \; \mathbb{E}_{s \sim d^\star} \log \hat Z^{(t)}(s)
   - \lambda^{(t)} \left( V_{g}^{\pi^\star}(\rho) - V_{g}^{(t)}(\rho) \right)+\frac{1+\lambda^{(t)}}{1-\gamma}\epsilon_A(M,H).
\end{align*}
In $(a)$, we use Lemma~\ref{Lemma:performance diff}.
For $(b)$, the result follows immediately from the definition.
In $(c)$, the inequality relies on the definition of the KL divergence together with Lemma~\ref{Lemma:bias advantage}.
Finally, step (d) again follows from Lemma~\ref{Lemma:performance diff}.

Taking the telescoping sum on both sides and dividing by $T$, we obtain
\begin{align*}
&\frac{1}{T} \sum_{t=0}^{T-1} \Big( V_{r}^{\pi^\star}(\rho) - V_{r}^{(t)}(\rho) \Big)\\
&= \frac{1}{\eta_1 T} \sum_{t=0}^{T-1} \mathbb{E}_{s \sim d^\star} 
   \Big[ D_{\mathrm{KL}} \!\left( \pi^\star(a \mid s) \,\|\, \pi^{(t)}(a \mid s) \right) 
        - D_{\mathrm{KL}} \!\left( \pi^\star(a \mid s) \,\|\, \pi^{(t+1)}(a \mid s) \right) \Big] \\
&\quad + \frac{1}{\eta_1 T} \sum_{t=0}^{T-1} \mathbb{E}_{s \sim d^\star} \log \hat Z^{(t)}(s) 
   - \frac{1}{T} \sum_{t=0}^{T-1} \lambda^{(t)} \Big( V_{g}^{\pi^\star}(\rho) - V_{g}^{(t)}(\rho) \Big)+\frac{1+\lambda^{(t)}}{1-\gamma}\epsilon_A(M,H) \\
&\overset{(a)}{\leq} \frac{1}{\eta_1 T} \sum_{t=0}^{T-1} \mathbb{E}_{s \sim d^\star} 
   \Big[ D_{\mathrm{KL}} \!\left( \pi^\star(a \mid s) \,\|\, \pi^{(t)}(a \mid s) \right) 
        - D_{\mathrm{KL}} \!\left( \pi^\star(a \mid s) \,\|\, \pi^{(t+1)}(a \mid s) \right) \Big] \\
&\quad + \frac{1}{(1-\gamma)T} \sum_{t=0}^{T-1} \Big( V_{r}^{(t+1)}(d^\star) - V_{r}^{(t)}(d^\star) \Big) + \frac{1}{(1-\gamma)T} \sum_{t=0}^{T-1} \lambda^{(t)} \Big( V_{g}^{(t+1)}(d^\star) - V_{g}^{(t)}(d^\star) \Big)\\
&\quad  
   - \frac{1}{T} \sum_{t=0}^{T-1} \lambda^{(t)} \Big( V_{g}^{\pi^\star}(\rho) - V_{g}^{(t)}(\rho) \Big) +2\frac{1+\lambda^{(t)}}{1-\gamma}\epsilon_A(M,H)\\
&\overset{(b)}{\leq} \frac{1}{\eta_1 T} \mathbb{E}_{s \sim d^\star} D_{\mathrm{KL}} \!\left( \pi^\star(a \mid s) \,\|\, \pi^{(0)}(a \mid s) \right) 
   + \frac{1}{(1-\gamma)T} V_{r}^{(T)}(d^\star)  \\
&\quad +  \frac{1}{T} \lambda^{(T)} V_g^{(T)}(\mu)
+ \frac{1}{T} \sum_{t=0}^{T-1} 
\left| \lambda^{(t)} - \lambda^{(t+1)} \right| V_g^{(t+1)}(\mu)\\
&\quad - \frac{1}{T} \sum_{t=0}^{T-1} \lambda^{(t)} \Big( V_{g}^{\pi^\star}(\rho) - V_{g}^{(t)}(\rho) \Big)+2\frac{1+\bar{\sigma}_\lambda}{1-\gamma}\epsilon_A(M,H)\\
&\overset{(c)}{\leq} \frac{1}{\eta_1 T} \mathbb{E}_{s \sim d^\star} D_{\mathrm{KL}} \!\left( \pi^\star(a \mid s) \,\|\, \pi^{(0)}(a \mid s) \right) 
   + \frac{1}{(1-\gamma)T} V_{r}^{(T)}(d^\star) + \frac{2 \eta_2}{(1-\gamma)^3} \\
&\quad - \frac{1}{T} \sum_{t=0}^{T-1} \lambda^{(t)} \Big( V_{g}^{\pi^\star}(\rho) - V_{g}^{(t)}(\rho) \Big)+2\frac{1+\bar{\sigma}_\lambda}{1-\gamma}\epsilon_A(M,H).
\end{align*}
In step~$(a)$, Lemma~\ref{Lemma:nonmonotone} is applied with $\mu = d^{\star}$. 
For step~$(b)$, non-positive terms are removed. 
Step~$(c)$ relies on the relations
\[
\bigl|\lambda^{(T)}\bigr| \le \frac{\eta_2 T}{1 - \gamma}
\quad \text{and} \quad
\bigl|\lambda^{(t)} - \lambda^{(t+1)}\bigr|
\le \frac{\eta_2}{1 - \gamma},
\]
which follow from the dual update in line~23 of Algorithm~\ref{alg:NPGPD-HF} 
and the non-expansiveness of projection, together with the bound 
\(V_g^{(t)}(\mu) \le \tfrac{1}{1 - \gamma}\). 

Finally, using the facts that 
\( D_{\mathrm{KL}}(p \,\|\, q) \le \log |A| \) 
for \( p, q \in \Delta_A \) with \( q \) being the uniform distribution, 
\( V_r^{(T)}(d^\star) \le \tfrac{1}{1 - \gamma} \), 
and \( V_g^\star(\rho) \ge b \),
the proof is complete.
\end{proof}

\textbf{We now bound the optimality gap in Theorem \ref{thm1}.}
Directly add Lemma \ref{Lemma:lambda} and \ref{Lemma:bounded_avg_perf}, we can get
\begin{align*}
&\frac{1}{T} \sum_{t=0}^{T-1} \Big( V_{r}^{\pi^\star}(\rho) - V_{r}^{(t)}(\rho) \Big)
\\
&\leq\frac{\log |A|}{\eta_1 T}
+ \frac{1}{(1-\gamma)^2 T}
+ \frac{2 \eta_2}{(1-\gamma)^3}+2\frac{1+
\bar{\sigma}_\lambda}{1-\gamma}\epsilon_A+\bar{\sigma}_\lambda\epsilon_3(H,M)+\frac{\sigma^2_3(H,M,N)\eta_2}{2}\\
&\leq \frac{4}{(1-\gamma)^2 \sqrt{T}}+\frac{\sigma^2_3(H,M,N)(1-\gamma)}{2\sqrt{T}}+2\frac{1+\bar{\sigma}_\lambda}{1-\gamma}\epsilon_A(M,H)+\bar{\sigma}_\lambda\epsilon_3(H,M)+\frac{\sigma^2_3(H,M,N)\eta_2}{2}\\
&= \left(\frac{4}{(1-\gamma)^3 }+\frac{3L^2 \log M}{M} + \frac{6G(H)^2}{M^2}+\frac{3G(H)^2}{N}+ \frac{3(\gamma^H)^2+1.5}{(1-\gamma)^2}\right)\frac{1-\gamma}{\sqrt{T}}\\
&+\frac{\sigma^2_3(H,M,N)(1-\gamma)}{2\sqrt{T}}+\frac{2+3\bar{\sigma}_\lambda-\gamma\bar{\sigma}_\lambda}{1-\gamma}\left(L\sqrt{ \frac{2 \log M}{M} } + \frac{2G(H)}{M^2}\right)+\frac{4+5\bar{\sigma}_\lambda-\gamma\bar{\sigma}_\lambda}{1-\gamma}\frac{\gamma^H}{1-\gamma},
\end{align*}
which completes the proof.

\textbf{We then bound the constraint violation in Theorem \ref{thm1}.}
By line 23 of Algorithm \ref{alg:NPGPD-HF}, for any $\lambda \in \left[0, \frac{2}{(1-\gamma)\xi}\right]$, 
\begin{align*}
\left|\lambda^{(t+1)} - \lambda\right|^2 
&\le 
\left| \lambda^{(t)} - \eta_2 \left(\hat{h}_c^{(\rho)}\right) - \lambda \right|^2 \\
&= 
\left|\lambda^{(t)} - \lambda\right|^2 
- 2\eta_2 \left(\hat{h}_c^{(t)}(\rho)\right)\left(\lambda^{(t)} - \lambda\right) 
+ \eta_2^2 \left(\hat{h}_c^{(t)}(\rho)\right)^2, \\
\end{align*}
where the first inequality is due to the non-expansiveness of projection operator $\mathcal{P}_\Lambda$ and the last inequality is due to 
$\left(V_g^{\boldsymbol{\theta}^{(t)}}(\rho) - b\right)^2 \le \frac{1}{(1-\gamma)^2}$.
Summing it up from $t = 0$ to $t = T-1$ and dividing it by $T$ yield
\[
0 \le \frac{1}{T}\left|\lambda^{(T)} - \lambda\right|^2 
\le \frac{1}{T}\left|\lambda^{(0)} - \lambda\right|^2 
- \frac{2\eta_2}{T} \sum_{t=0}^{T-1} 
\left(\hat{h}_c^{(t)}(\rho) \right)\left(\lambda^{(t)} - \lambda\right) 
+ \eta_2^2(\hat{h}_c^{(t)}(\rho))^2.
\]
Take expectation and use Lemma \ref{Lemma:var hc}, we can get
\[
\frac{1}{T} \sum_{t=0}^{T-1} 
\mathbb{E}\left(\hat{h}_c^{(t)}(\rho) \right)\left(\lambda^{(t)} - \lambda\right)
\le 
\frac{1}{2\eta_2 T}\left|\lambda^{(0)} - \lambda\right|^2 
+ \frac{\eta_2}{2}\sigma_2^2(H,M,N).
\]
Add it to Lemma \ref{Lemma:bounded_avg_perf},
\begin{align*}
&\frac{1}{T} \sum_{t=0}^{T-1} \Big( V_{r}^{\pi^\star}(\rho) - V_{r}^{(t)}(\rho) \Big)
+ \frac{1}{T} \sum_{t=0}^{T-1} \lambda^{(t)} \left( V_{g}^{\pi^\star}(\rho) - V_{g}^{(t)}(\rho)+h_c^{(t)}(\rho) \right)-\frac{\lambda}{T}\sum_{t=0}^{T-1}h_c^{(t)}(\rho)\\
&=\frac{1}{T} \sum_{t=0}^{T-1} \Big( V_{r}^{\pi^\star}(\rho) - V_{r}^{(t)}(\rho) \Big)\\
&\quad+ \frac{1}{T} \sum_{t=0}^{T-1} \lambda^{(t)} \left( V_{g}^{\pi^\star}(\rho) - V_{g}^{(t)}(\rho)+V_{g}^{(t)}(\rho)-b+h_c^{(t)}(\rho)-\left(V_{g}^{(t)}(\rho)-b\right) \right)\\
&\quad-\frac{\lambda}{T}\sum_{t=0}^{T-1}\left(h_c^{(t)}(\rho)-\left(V_{g}^{(t)}(\rho)-b\right)+\left(V_{g}^{(t)}(\rho)-b\right)\right).
\end{align*}
By taking expectations on both sides and applying Lemmas \ref{Lemma:bias hc} and \ref{Lemma:bias advantage}, together with the facts that $\lambda^{(t)} \leq \bar{\sigma}_{\lambda}$ and $V_{g}^{\pi^\star}(\rho) - b > 0$, we obtain
\begin{align*}
&\quad\frac{1}{T} \sum_{t=0}^{T-1} \mathbb{E}\left( V_{r}^{\pi^\star}(\rho) - V_{r}^{(t)}(\rho) \right)- 2\bar{\sigma}_{\lambda} \epsilon_c(H,M)+\frac{\lambda}{T}\sum_{t=0}^{T-1}\mathbb{E}\left(b-V_{g}^{(t)}(\rho)\right)\\
&\leq\frac{\log |A|}{\eta_1 T}
+ \frac{1}{(1-\gamma)^2 T}
+ \frac{2 \eta_2}{(1-\gamma)^3}+2\frac{1+\bar{\sigma}_\lambda}{1-\gamma}\epsilon_A(M,H)\\
&\quad+\frac{1}{2\eta_2 T}\left|\lambda^{(0)} - \lambda\right|^2 
+ \frac{\eta_2}{2}\sigma_2^2(H,M,N).
\end{align*}
When $\sum_{t=0}^{T-1} \mathbb{E}\!\left[b - V_{g}^{(t)}(\rho)\right] \ge 0$, we set $\lambda = \bar{\sigma}_{\lambda}$; otherwise, we set $\lambda = 0$.
Thus,
\begin{align*}
&V_{r}^{\pi^\star}(\rho)- \mathbb{E}\left( \frac{1}{T} \sum_{t=0}^{T-1} V_{r}^{(t)}(\rho) \right)- 2\bar{\sigma}_{\lambda} \epsilon_c(H,M)+\bar\sigma_\lambda\left[b-\mathbb{E}\left(\frac{1}{T}\sum_{t=0}^{T-1}V_{g}^{(t)}(\rho)\right)\right]_+\\
&\leq\frac{\log |A|}{\eta_1 T}
+ \frac{1}{(1-\gamma)^2 T}
+ \frac{2 \eta_2}{(1-\gamma)^3}+2\frac{1+\bar{\sigma}_\lambda}{1-\gamma}\epsilon_A(M,H)\\
&\quad+\frac{1}{2\eta_2 T}\left|\lambda^{(0)} - \lambda\right|^2 
+ \frac{\eta_2}{2}\sigma_2^2(H,M,N)
\end{align*}
As in section \ref{proof:thm2},we know that $V^{(t)}_r(\rho)$ and $V^{(t)}_g(\rho)$ are linear functions of the occupancy measure $q^{\pi}$ induced by policy $\pi^{(t)}$.
By the convexity of the set of occupancy measures, the average of $T$ occupancy measures and take expectation is still an occupancy measure that induces a policy $\pi'$ with values $V^{\pi'}_r$ and $V^{\pi'}_g$.
Therefore, there exists a policy $\pi'$ such that 
\[
V^{\pi'}_r(\rho) =\mathbb{E} \left(\frac{1}{T} \sum_{t=0}^{T-1} V^{(t)}_r(\rho)\right), 
\quad 
V^{\pi'}_g(\rho) = \mathbb{E} \left(\frac{1}{T} \sum_{t=0}^{T-1} V^{(t)}_g(\rho)\right).
\]

\medskip

Therefore,
\begin{align*}
&V_r^\star(\rho) - V^{\pi'}_r(\rho) + \bar{\sigma}_{\lambda} \big[ b - V^{\pi'}_g(\rho) \big]_+\\
& 
\leq\frac{\log |A|}{\eta_1 T}
+ \frac{1}{(1-\gamma)^2 T}
+ \frac{2 \eta_2}{(1-\gamma)^3}+2\frac{1+\bar{\sigma}_\lambda}{1-\gamma}\epsilon_A(M,H)\\
&\quad+\frac{1}{2\eta_2 T}\left|\lambda^{(0)} - \lambda\right|^2 
+ \frac{\eta_2}{2}\sigma_2^2(H,M,N).
\end{align*}

We note that $\frac{2}{(1-\gamma)\xi} \ge 2\lambda^\star$. 
According to Lemma~\ref{Lemma:constraint_violation}, we obtain
\begin{align*}
&\big[ b - V^{\pi'}_g(\rho) \big]_+\\
& 
\leq\frac{2}{\bar{\sigma}_{\lambda}}\left(\frac{\log |A|}{\eta_1 T}
+ \frac{1}{(1-\gamma)^2 T}
+ \frac{2 \eta_2}{(1-\gamma)^3}+2\frac{1+\bar{\sigma}_\lambda}{1-\gamma}\epsilon_A(M,H)\right.\\
&\left.\quad+\frac{1}{2\eta_2 T}\left|\lambda^{(0)} - \lambda\right|^2 
+ \frac{\eta_2}{2}\sigma_2^2(H,M,N)\right)\\
& 
\leq\frac{2}{\bar{\sigma}_{\lambda}}\left(\frac{\log |A|}{\eta_1 T}
+ \frac{1}{(1-\gamma)^2 T}
+ \frac{2 \eta_2}{(1-\gamma)^3}+2\frac{1+\bar{\sigma}_\lambda}{1-\gamma}\left(L\sqrt{\frac{2 \log M}{M}} + \frac{2G}{M^2} + 2\frac{\gamma^H}{1-\gamma}\right)\right.\\
&\left.\quad+\frac{1}{2\eta_2 T}\left|\lambda^{(0)} - \lambda\right|^2 
+ \frac{\eta_2}{2}\sigma_2^2(H,M,N)\right).
\end{align*}
By taking $\eta_1 = 2\log|A|$ and $\eta_2 = \frac{1-\gamma}{\sqrt{T}}$ and noticing the fact that $\frac{1}{2\eta_2 T}\left|\lambda^{(0)} - \lambda\right|^2\leq\mathcal{O}(\frac{1}{\sqrt{T}})$, we can get the similar result as the optimality gap:
\begin{equation*}
\frac{1}{T} \sum_{t=0}^{T-1} 
\mathbb{E}\!\left[ b - V_{g}^{(t)}(\rho) \right]_+
\;\leq\; 
\mathcal{O}\!\left( \frac{1}{\sqrt{T}} \;+\; L \sqrt{\frac{ \log M}{M}} \;+\;  \frac{\gamma^H}{1-\gamma} \right).
\end{equation*}

\end{document}